\documentclass{article}

\def\showdraftbox{0}
\def\showauthornotes{0}

\usepackage{amstext,amsfonts,bbm,algorithm,algorithmicx,xspace,nicefrac}
\usepackage{color,stmaryrd,enumerate,latexsym,bm,amsfonts,
  wrapfig,verbatim,textcomp}
\usepackage[small]{caption}
\usepackage{comment} 
\usepackage{epsfig} 
\usepackage{tabularx}
\usepackage{tabu}
\usepackage{fancybox,graphicx,url}
\usepackage{enumitem, fullpage}
\usepackage{booktabs}
\usepackage{commath}

\newcommand{\defeq}{\stackrel{\textup{def}}{=}}

\newtheorem{theorem}{Theorem}[section]
\newtheorem{lemma}[theorem]{Lemma}
\newtheorem{definition}[theorem]{Definition} 
\newtheorem{corollary}[theorem]{Corollary}

\newtheorem{proposition}[theorem]{Proposition}

\def\abs#1{\left| #1 \right|}
\renewcommand{\norm}[1]{\ensuremath{\left\lVert #1 \right\rVert}}

\newcommand\rea{\mathbb R}

\newcommand{\marginlabel}[1]%
{\mbox{}\marginpar{\it{\raggedleft\hspace{0pt}#1}}}

\DeclareMathOperator*{\argmin}{arg\,min}
\DeclareMathOperator*{\argmax}{arg\,max}

\definecolor{Mygray}{gray}{0.8}

 \ifcsname ifcommentflag\endcsname\else
  \expandafter\let\csname ifcommentflag\expandafter\endcsname
                  \csname iffalse\endcsname
\fi

\ifnum\showauthornotes=1

\else

\fi

\ifnum\showauthornotes=1
\newcommand{\Authornote}[2]{{\sf\color{red}{[#1: #2]}}}
\newcommand{\Authoredit}[2]{{\sf\color{red}{[#1]}\color{blue}{#2}}}
\newcommand{\Authorcomment}[2]{{\sf \color{gray}{[#1: #2]}}}
\newcommand{\Authorfnote}[2]{\footnote{\color{red}{#1: #2}}}
\newcommand{\Authorfixme}[1]{\Authornote{#1}{\textbf{??}}}
\newcommand{\Authormarginmark}[1]{\marginpar{\textcolor{red}{\fbox{
#1:!}}}}
\else
\newcommand{\Authornote}[2]{}
\newcommand{\Authoredit}[2]{}
\newcommand{\Authorcomment}[2]{}
\newcommand{\Authorfnote}[2]{}
\newcommand{\Authorfixme}[1]{}
\newcommand{\Authormarginmark}[1]{}
\fi

\newlength{\pgmtab}  
\newcommand {\ELSE}{{\bf else\ }}
\newcommand {\IF}{{\bf if\ }}
\newcommand {\FOR}{{\bf for\ }}
\newcommand {\TO}{{\bf to\ }}

\newcommand {\WHILE}{{\bf while\ }}

\newcommand {\THEN}{\mbox{\bf then\ }}

\newcommand {\RETURN}{\mbox{\bf return\ }}

\newcommand {\TRUE}{\mbox{\bf true\ }}
\newcommand {\FALSE}{\mbox{\bf false\ }}

 {
	\begin{enumerate}}{\end{enumerate}}

\def\qedsketch{\ifmmode\Box\else{\unskip\nobreak\hfil
\penalty50\hskip1em\null\nobreak\hfil$\Box$
\parfillskip=0pt\finalhyphendemerits=0\endgraf}\fi}

\newenvironment{proof}{\begin{trivlist} \item {\bf Proof:~~}}
   {\hfill $\Box$ \end{trivlist}}

\newenvironment{proofof}[1]{\begin{trivlist} \item {\bf Proof
#1:~~}}
  { \hfill $\Box$ \end{trivlist}}

\newlength{\tpush}
\newcommand{\handout}[5]{
   \noindent
   \begin{center}
   \framebox{ \vbox{ \hbox to \textwidth { {\bf \coursenum\ :\  \coursename} \hfill #5 }
       \vspace{3mm}
       \hbox to \textwidth { {\Large \hfill #2  \hfill} }
       \vspace{1mm}
       \hbox to \textwidth { {\it #3 \hfill #4} }
     }
   }
   \end{center}
   \vspace*{4mm}
   \newcommand{\lecturenum}{#1}
   \addcontentsline{toc}{chapter}{Lecture #1 -- #2}
}

\ifnum\showdraftbox=1

\else

\fi

\usepackage{alltt} 

\newtheorem{observation}[theorem]{Observation}
\newtheorem{claim}[theorem]{Claim}

\newdimen\pIR
\newcommand\StevesR{{\rm I\kern\pIR R}}
\def\Reals#1{\StevesR^{#1}}

\def\defeq{=}
\def\setof#1{\left\{#1  \right\}}
\def\sizeof#1{\left|#1  \right|}

\def\union{\cup}
\def\intersect{\cap}

\def\abs#1{\left|#1  \right|}

\def\norm#1{\left\| #1 \right\|}

\newcommand\dd{\boldsymbol{\mathit{d}}}

\usepackage{caption,mdframed}

\usepackage{times}

\allowdisplaybreaks[1]

\usepackage{natbib}
\usepackage{hyperref}
\usepackage{nameref}
\hypersetup{colorlinks,
            linkcolor=blue,
            citecolor=blue,
            urlcolor=magenta,
            linktocpage,
            plainpages=false}

\bibpunct{(}{)}{;}{a}{,}{,}

\newcommand{\Anote}{\Authornote{A}}

\newcommand{\eps}{\epsilon}

\newcommand\len{\ensuremath{\ell}}
\newcommand\fixpath{\ensuremath{\mathsf{fix}}}
\newcommand\grad{\ensuremath{\mathsf{grad}}}
\newcommand\lex{\ensuremath{\mathsf{lex}}}

\newcommand\lexless{\preceq}

\newcommand\nlexless{\not\preceq}

\newcommand{\interior}{\operatorname{int}}

\newcommand{\dis}{{\mathsf{dist}}}

\newcommand{\meta}{{\sc Meta-Lex}}

\newcommand{\nbr}[1]{\left\|#1\right\|}

\newenvironment{tight_enumerate}{
\begin{enumerate}
  \setlength{\itemsep}{0pt}
  \setlength{\topsep}{30pt}
  \setlength{\parskip}{0pt}
}{\end{enumerate}}

\newcommand\Ehat{{\widehat{{E}}}}

\newcommand{\terminalpressure}{{\sc Term-Pressure}}
\newcommand{\tcdagvc}{\textsc{Konig-Cover}}
\newcommand{\outlieralg}{\textsc{Outlier}}
\newcommand{\approxoutlieralg}{\textsc{Approx-Outlier}}

\newcommand\Vhat{{\widehat{{V}}}}

\begin{document}

\title{Algorithms for Lipschitz Learning on Graphs
  \thanks{This research was partially supported by
   AFOSR Award FA9550-12-1-0175,
   NSF grant CCF-1111257, a Simons Investigator Award to Daniel Spielman, and a MacArthur Fellowship.}
    \thanks{Code used in this work is available at \texttt{https://github.com/danspielman/YINSlex}}
    }

 \author{
 Rasmus Kyng \\
 Yale University\\
 \texttt{rasmus.kyng@yale.edu}
\and
 Anup Rao \\
 Yale University\\
  \texttt{anup.rao@yale.edu}
  \and
Sushant Sachdeva\\
 Yale University \\
 \texttt{sachdeva@cs.yale.edu}
  \and
Daniel A. Spielman\\
Yale University\\
 \texttt{spielman@cs.yale.edu}
 }

\maketitle


\begin{abstract}
  We develop fast algorithms for solving regression problems on graphs
  where one is given the value of a function at some vertices, and
  must find its smoothest possible extension to all vertices.
  The extension we compute is the absolutely minimal
  Lipschitz extension, and is the limit for large $p$ of $p$-Laplacian
  regularization.  
  We present an algorithm that computes a minimal Lipschitz extension
  in expected linear time, and an algorithm that computes an absolutely
  minimal Lipschitz extension in expected time $\widetilde{O} (m n)$.
  The latter algorithm has variants that seem to run much faster in practice.
  These
  extensions are particularly amenable to regularization: we can
  perform $l_{0}$-regularization on the given values in polynomial
  time and $l_{1}$-regularization on the initial function values and on graph edge weights in time
  $\widetilde{O} (m^{3/2})$.
  
  Our definitions and algorithms naturally extend to
  directed graphs.
\end{abstract}






\section{Introduction}


We consider a problem in which we are given a weighted undirected graph $G = (V,E,\len)$ and values $v_{0}: T \to \mathbb{R}$
  on a subset $T$ of its vertices.  
We view the weights $\len$ as indicating the lengths of edges, with shorter length
  indicating greater similarity.
Our goal it to assign values to every vertex
  $v \in V \backslash T$ so that the values assigned are as {\it smooth} as possible across edges.
A minimal Lipschitz extension of $v_{0}$ is a vector $v$ that minimizes
\begin{equation}\label{eqn:inf}
  \max_{(x,y) \in E}  (\len{(x,y)})^{-1} \abs{v (x) - v (y)} ,
\end{equation}
subject to $v (x) = v_{0} (x)$ for all  $x \in T$.
We call such a vector an inf-minimizer.
Inf-minimizers are not unique.
So, among inf-minimizers
  we seek vectors that minimize the second-largest absolute value of
  $ \len{(x,y)}^{-1} \abs{v (x) - v (y)} $
  across edges,
  and then the third-largest given that, and so on.
We call such a vector $v$ a lex-minimizer.
It is also known as an absolutely minimal Lipschitz extension of $v_{0}$.

These are the limit of the solution to $p$-Laplacian minimization problems for large $p$, namely the vectors that solve
 \begin{align}
\label{eq:pLap}
 \min_{\substack{v \in \mathbb{R}^n \\ v |_T = v_{0}|_{T}}} 
   \sum_{(x,y) \in E} (\len{(x,y)})^{-p} |v(x) - v(y)|^p. 
 \end{align}
The use of $p=2$ was suggested in the foundational paper of \cite{Zhu03},
  and is particularly nice because it can be obtained by solving a 
  system of linear equations in a symmetric diagonally dominant matrix, which can be done very quickly
  (\cite{cohen2014solving}).
The use of larger values of $p$ has been discussed
  by \cite{Ulrike_p_resist}, and by \cite{NBridle}, but it is much more complicated to compute.
The fastest algorithms we know for this problem require convex programming, and
  then require very high accuracy to obtain the values at most vertices.
By taking the limit as $p$ goes to infinity, we recover the lex-minimizer, which we will show can be computed quickly.

The lex-minimization problem has a remarkable amount of structure.
For example, in uniformly weighted graphs the value of the
lex-minimizer at every vertex not in $T$ is equal to the average of
the minimum and maximum of the values at its neighbors.  This is
analogous to the property of the $2$-Laplacian minimizer that the
value at every vertex not in $T$ equals the average of the values at
its neighbors.

\subsection{Contributions}
We first present several important structural properties of
lex-minimizers in Section \ref{ssec:characterizations}. As we shall
point out, some of these were known from previous work, sometimes in
restricted settings. We state them generally and prove them for
completeness. We also prove that the lex-minimizer is as stable as
possible under perturbations of $v_{0}$ (Section
\ref{ssec:stability}).

The structure of the lex-minimization problem has led us to develop elegant algorithms
  for its solution.
Both the algorithms and their analyses could be taught to undergraduates.
We believe that these algorithms could be used in place of $2$-Laplacian minimization
  in many applications.

We present algorithms for the following problems.  Throughout, $m =
\sizeof{E}$ and $n = \sizeof{V}$.
 \begin{description}
 \item[Inf-minimization:] An algorithm that runs in expected time $O(m
   + n \log n)$  (Section \ref{sec:infAlg}).
 \item[Lex-minimization:] An algorithm that runs in expected time 
 $O(n (m + n \log n))$ (Section \ref{sec:algs}), along with a variant
 that runs quickly in practice (Section \ref{sec:fastAlg}).
  \item[$l_1$-regularization of edge lengths for inf-minimization:]
  The problem of minimizing \eqref{eqn:inf} 
  given a limited budget with which one can increase edge lengths is a linear programming problem.  
  We show how to solve it in time $\widetilde{O} (m^{3/2})$ with an interior point method
  by using fast Laplacian solvers (Section~\ref{sec:l1_reg}).
  The same algorithm can accommodate $l_{1}$-regularization of the values given in $v_{0}$.

 \item[$l_0$-regularization of vertex values for inf-minimization:] 
  We give a polynomial time algorithm for
  $l_{0}$-regularization of the values at vertices.
  That is, we minimize \eqref{eqn:inf} given a budget of a number of vertices that can
  be proclaimed outliers and removed from $T$ (Section~\ref{sec:l0reg}).
  We solve this problem by reducing it to the problem of computing minimum vertex
  covers on transitively closed directed acyclic graphs, a special case of minimum vertex cover that can be solved in polynomial time.
 \end{description}
After any regularization for inf-minimization, we suggest computing the lex-minimizer.
We find the result for $l_{0}$-regularization of vertex values to be particularly surprising,
  especially because we prove that the analogous problem for $2$-Laplacian minimization
  is NP-Hard (Section~\ref{sec:l2hardness}).

  All of our algorithms extend naturally to \textbf{directed graphs} (Section~\ref{sec:directed}).
  This is in contrast with the problem of minimizing $2$-Laplacians on
  directed graphs, which corresponds to computing electrical flows in
  networks of resistors and diodes, for which fast algorithms are not
  presently known.

  We present a few \textbf{experiments} on examples demonstrating that
  the lex-minimizer can overcome known deficiencies of the
  $2$-Laplacian minimizer (Section~\ref{sec:related},
  Figures~\ref{fig:gauss},\ref{fig:cube4}), as well as a demonstration of the
  performance of the directed analog of our algorithms on the WebSpam
  dataset of \cite{Castillo2006} (Section~\ref{sec:experiments}).  In the WebSpam problem we use the
  link structure of a collection of web sites to flag some sites as
  spam, given a small number of labeled sites known to be spam or
  normal.

\subsection{Relation to Prior Work}
\label{sec:related}

We first encountered the idea of 
  using the minimizer of the 2-Laplacian given by  \eqref{eq:pLap}
  for regression and classification on graphs in the work of \cite{Zhu03}
  and \cite{BelkinRegression} on semi-supervised learning.
These works transformed learning problems on sets of vectors into problems on graphs
  by identifying vectors with vertices and constructing graphs with edges between
  nearby vectors.
One shortcoming of this approach (see \cite{Nadler}, \cite{Ulrike_p_resist}, \cite{NBridle})
  is that if the number of vectors grows while the number of labeled vectors remains fixed,
  then almost all the values of the 2-Laplacian minimizer 
  converge to the mean of the labels on most natural examples.
For example, \cite{Nadler} consider sampling
  points from two Gaussian distributions centered at $0$ and $4$ on the real line.
They place edges between every pair of points $(x,y)$ with length $\exp (\abs{x-y}^{2}/2 \sigma^{2})$
  for $\sigma = 0.4$, and provide only the labels $v_{0} (0) = -1$ and $v_{0} (4) = 1$.
Figure~\ref{fig:gauss} shows the values of the $2$-Laplacian
  minimizer  in red, which are all approximately zero.
In contrast, the values of the lex-minimizer  in blue,  which are smoothly distributed between the labeled points,
  are shown.

The ``manifold hypothesis'' (see \cite{ChaSchZie06}, \cite{Ma:2011:MLT:2207974})
  holds that much natural data lies near a low-dimensional manifold and that
  natural functions we would like to learn on this data are smooth functions on the manifold.
Under this assumption, one should expect lex-minimizers to interpolate well.
In contrast, the $2$-Laplacian minimizers degrade (dotted lines) if the number of labeled
  points remains fixed while the total number of points grows.
In Figure~\ref{fig:cube4}, we demonstrate this by sampling many points uniformly from the
  unit cube in 4 dimensions, form their 8-nearest neighbor graph, 
  and consider the problem of regressing the first coordinate.
We performed 8 experiments, varying the number of labeled points in $\setof{50, 100, 500, 1000}$.
Each data point is the mean average $l_{1}$ error over 100 experiments.
The plots for root mean squared error are similar.
The standard deviation of the estimations of the mean are within one pixel, and so are not displayed.
The performance of the lex-minimizer (solid lines) does not degrade as the number of unlabeled points grows.

\begin{figure}
\centering
\begin{minipage}{.4\textwidth}
  \centering
  \includegraphics[width=\textwidth]{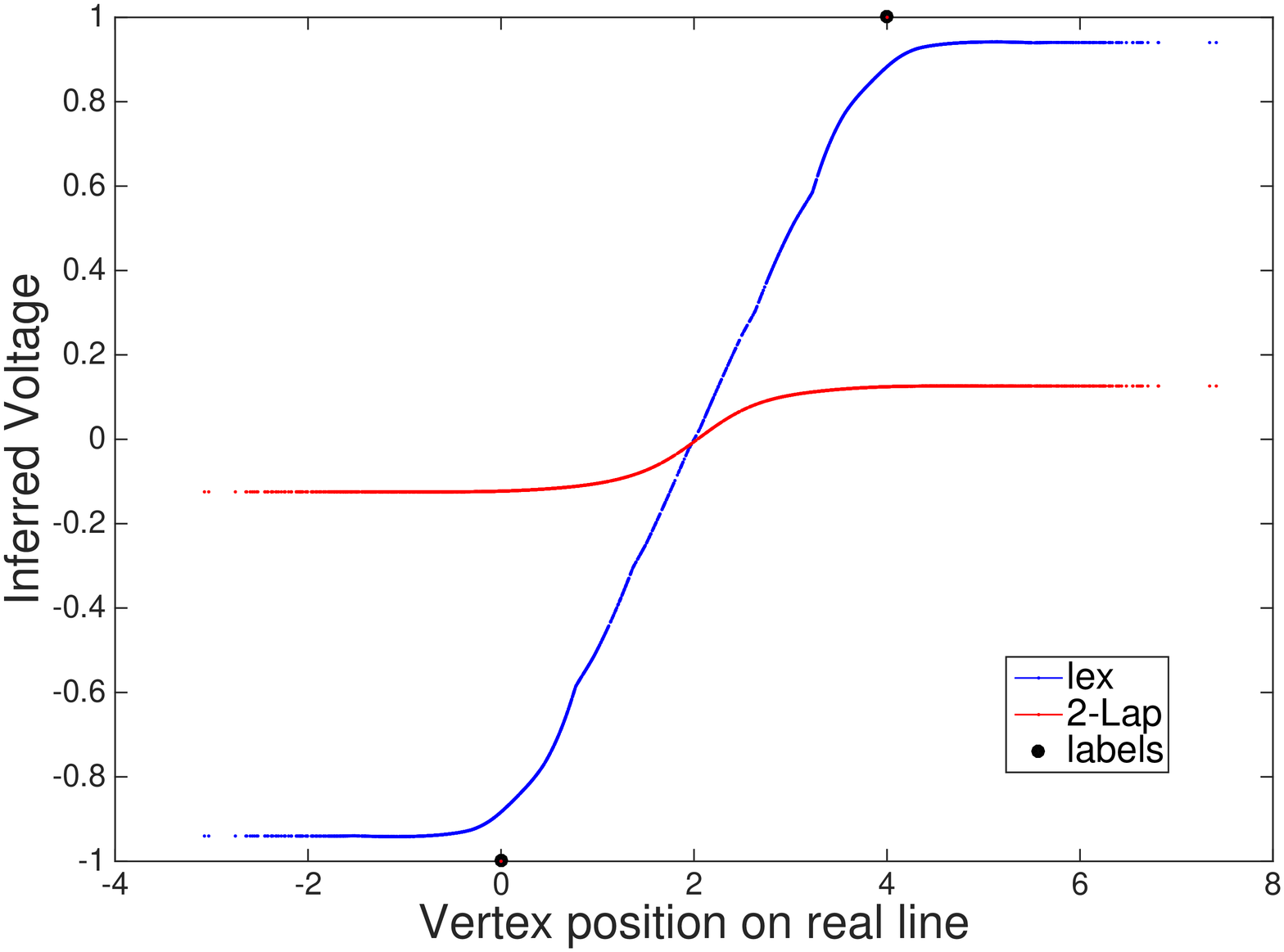}
  \captionof{figure}{Lex vs 2-Laplacian on 1D gaussian clusters.}
  \label{fig:gauss}
\end{minipage}%
\,\,\,
\begin{minipage}{.4\textwidth}
  \centering
  \includegraphics[width=\textwidth]{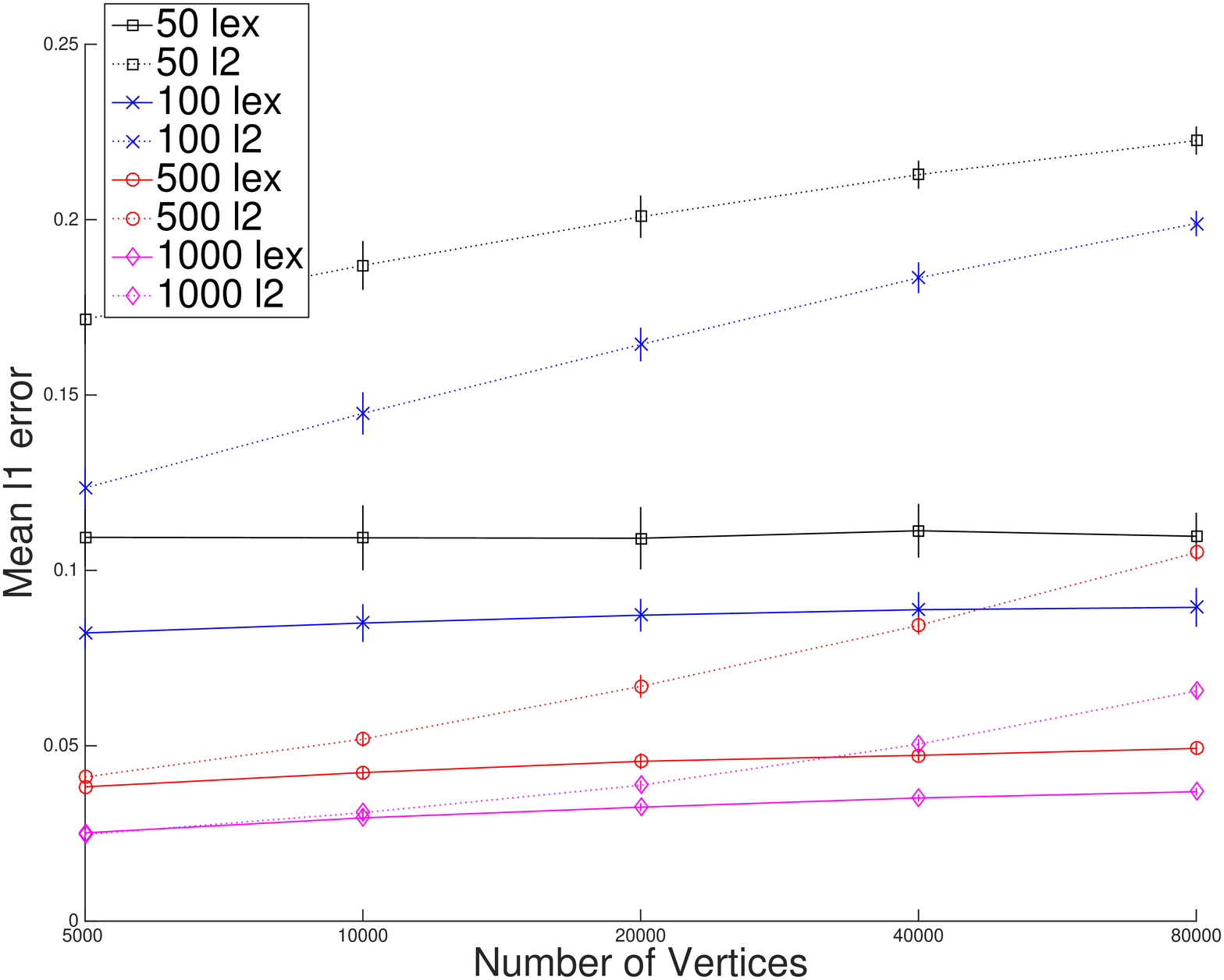}
  \captionof{figure}{{kNN graphs on samples from 4D cube.}
}
  \label{fig:cube4}
\end{minipage}

\end{figure}


Analogous to our inf-minimizers, minimal Lipschitz extensions of
functions in Euclidean space and over more general metric spaces have
been studied extensively in Mathematics (\cite{Kirszbraun},
\cite{McShane}, \cite{Whitney}). \cite{UlrikeLip} employ Lipschitz
extensions on metric spaces for classification and relate these to
Support Vector Machines.  Their work inspired improvements in
classification and regression in metric spaces with low doubling
dimension (\cite{KrauthgamerClassification},
\cite{KrauthgamerRegression}).  Theoretically fast, although not
actually practical, algorithms have been given for constructing
minimal Lipschitz extensions of functions on low-dimensional Euclidean
spaces (\cite{Fefferman}, \cite{Fefferman1}, \cite{Fefferman3}).
\cite{SinopGrady} suggest using inf-minimizers for binary
classification problems on graphs.  For this special case, where all
of the given values are either 0 or 1, they present an
$O(m + n \log n)$ time algorithm for computing an inf-minimizer.  The
case of general given values, which we solve in this paper, is much
more complicated.  To compensate for the non-uniqueness of
inf-minimizers, they suggest choosing the inf-minimizer that minimizes
\eqref{eq:pLap} with $p = 2$.  We believe that the lex-minimizer is a
more natural choice.

The analog of our lex-minimizer over continuous spaces is called the
absolutely minimal Lipschitz extension (AMLE).  Starting with the work
of \cite{Aronsson}, there have been several characterizations and
proofs of the existence and uniqueness of the AMLE (\cite{Jensen},
\cite{Crandall}, \cite{Barles01}, \cite{Aronsson_atour}).  Many of
these results were later extended to general metric spaces, including
graphs (\cite{Milman}, \cite{Peres}, \cite{Naor}, \cite{Smart}).
However, to the best of our knowledge, fast algorithms for computing
lex-minimizers on graphs were not known.  For the special case of
undirected, unweighted graphs, \cite{Lazarus} presented both a
polynomial-time algorithm and an iterative method. \cite{Oberman}
suggested computing the AMLE in Euclidean space by first discretizing
the problem and then solving the corresponding graph problem by an
iterative method.  However, no run-time guarantees were obtained for
either iterative method.


\section{Notation and Basic Definitions}
\label{sec:lex_basics}
\paragraph{Lexicographic Ordering.}
Given a vector $r \in \rea^m,$ let $\pi_r$ denote a permutation that
sorts $r$ in non-increasing order by absolute value, \emph{i.e.},
$\forall i \in [m-1], |r(\pi_r(i))| \ge |r(\pi_r(i+1))|.$
Given two vectors $r, s \in \rea^m,$ we write $r \lexless s$ to indicate
that $r$ is smaller than $s$ in the \emph{lexicographic ordering} on
sorted absolute values,  \emph{i.e.}
\begin{align*}
  & \exists j \in [m],  \abs{r(\pi_{r}(j))} < \abs{s(\pi_{s}(j))} \text{ and } \forall i\in[j-1], \abs{r(\pi_{r}(i))} = \abs{s(\pi_{s}(i))} \\
  \text{ or } 
  & \forall i \in [m], \abs{r(\pi_{r}(i))} = \abs{s(\pi_{s}(i))}.
\end{align*}
Note that it is possible that 
  $r \lexless s$ and $s \lexless r$ while $r \neq s$. 
It is a
total relation: for every $r$ and $s$ at least one of
$r \lexless s$ or $s \lexless r$ is true.

\paragraph{Graphs and Matrices.} 
We will work with weighted graphs. Unless explicitly stated, we will
assume that they are undirected. For a graph $G$, we let $V_{G}$ be
its set of vertices, $E_{G}$ be its set of edges, and
$\len_{G} : E_{G} \to \rea_+$ be the assignment of positive lengths to
the edges.  We let $|V_{G}| =n,$ and $|E_{G}| = m.$
We assume $\len_{G}$ is
  symmetric,
\emph{i.e.}, $\len_{G}(x,y) = \len_{G}(y,x).$ 
When $G$ is clear from the
  context, we drop the subscript.

A \textbf{path} $P$ in $G$ is an ordered sequence of (not necessarily
distinct) vertices $P = (x_{0},x_1,\ldots,x_k),$ such that
  $(x_{i-1},x_i) \in E$ for $i \in [k].$ 
The \textbf{endpoints} of $P$ are
denoted by $\partial_{0} P = x_{0}, \partial_1 P = x_k.$ 
The set of \textbf{interior} vertices of $P$ is defined to be
$\interior(P) \defeq \{x_i : 0 < i < k\}.$ 
For $0 \le i < j \le k,$ we
use the notation $P[x_i:x_j]$ to denote the subpath
$(x_i,\ldots,x_j).$ The length of $P$ is 
$\len(P) \defeq \sum_{i=1}^k \len(x_{i-1},x_i).$

A function $v_{0} : V \to \rea \cup \{*\}$ is called a 
  \textbf{voltage assignment}
 (to $G$). A vertex $x \in V$ is a \textbf{terminal} with
respect to $v_{0}$ iff $v_{0}(x) \neq *.$ 
The other vertices, for which $v_{0} (x) = *$, are \textbf{non-terminals}.  
We let $T(v_{0})$ denote the set of terminals with respect to $v_{0}.$ If
$T(v_{0}) = V,$ we call $v_{0}$ a \textbf{complete voltage assignment} (to $G$). We
say that an assignment $v : V \to \rea \cup \{*\}$ \textbf{extends} $v_{0}$ if
$v(x) = v_{0}(x)$ for all $x$ such that $v_{0}(x) \neq *$.

Given an assignment $v_{0} : V \to \rea \cup \{*\}$, and 
  two terminals $x,y \in T(v_{0})$ for which $(x,y) \in E$,
  we define the \textbf{gradient} on $(x,y)$ due to $v_{0}$ to be \vspace{-6pt}
\[ 
\vspace{-6pt}
\grad_{G}[v_0](x,y) = \left. \frac{v_0(x)-v_0(y)}{\len(x,y)} \right. .
\] 
It may be useful to view
   $\grad_{G}[v_0](x,y)$ as the \emph{current} in the
edge $(x,y)$ induced  by voltages $v_{0}$.
When $v_0$ is a complete voltage assignment, we interpret $\grad_{G}[v_0]$ as a vector in
$\rea^m,$ with one entry for each edge.
However,
for convenience, we define $\grad_{G}[v_{0}](x,y) = -\grad_{G}[v_{0}](y,x).$ 
When $G$ is clear from the context, we drop the subscript.

A graph $G$ along with a voltage assignment $v$ to $G$ is called a
\textbf{partially-labeled graph}, denoted $(G,v).$
We say that a partially-labeled graph $(G,v_{0})$ is a 
  \textbf{well-posed instance} if for every
  maximal connected component
 $H$ of $G,$ we have
  $T(v_{0}) \cap V_H \neq \emptyset.$

A path $P$ in a partially-labeled graph $(G,v_{0})$ is
called a \textbf{terminal path} if
 both endpoints are terminals.
We define $\nabla P(v_{0})$ to be its gradient:  \vspace{-6pt}
\[
\vspace{-6pt}
\nabla P(v_{0}) \defeq \left. \frac{v_{0}(\partial_{0} P) - v_{0}(\partial_1
  P)}{\len(P)}\right. .
\]
If $P$ contains no terminal-terminal edges (and hence, contains at least one non-terminal), 
  it is a \textbf{free terminal path}.



\paragraph{Lex-Minimization.}
An instance of the  {\sc Lex-Minimization} problem is described by a
partially-labeled graph $(G,v_{0}).$
The objective is to compute a complete voltage assignment
$v : V_{G} \to \rea$ extending $v_{0}$ that lex-minimizes $\grad[v].$
\begin{definition}[Lex-minimizer]
\label{def:lex-min}
Given a partially-labeled graph $(G,v_{0}),$
we define $\lex_G[v_{0}]$ to be a
complete voltage assignment to $V$ that
  extends $v_{0},$ and such that
  for every other complete assignment $v^\prime : V_{G} \to \rea$ that
  extends $v_{0},$ we have
  $\grad_{G}[\lex_G[v_{0}]] \lexless \grad_{G}[v^\prime].$ That is,
  $\lex_G[v_{0}]$ achieves a lexicographically-minimal gradient
  assignment to the edges.
\end{definition}
\noindent We call $\lex_{G}[v_{0}]$ the lex-minimizer for
$(G,v_{0})$. Note that if $T(v_{0}) = V_{G},$ then trivially,
$\lex_G[v_{0}] = v_{0}.$
\vspace{-10pt}

\section{Basic Properties of Lex-Minimizers}
\cite{Lazarus} established that lex-minimizers in unweighted and undirected 
  graphs exist, are unique, and may be computed by an elementary meta-algorithm.
We state and prove these facts for undirected weighted graphs, and
defer the discussion of the directed case to Section~\ref{sec:directed}.
We also state for directed and weighted graphs
characterizations of lex-minimizers that were established by
\cite{Peres}, \cite{Naor} and \cite{Smart} for unweighted graphs.
These results are essential for the analyses of our algorithms.  We
defer most proofs to Appendix \ref{app:lex_basics}.
\begin{definition}
  A \emph{steepest fixable path} in an instance $(G,v_{0})$ is a free
  terminal path $P$  that has the largest gradient $\nabla P(v_{0})$
  amongst such paths. 
\end{definition}
Observe that a steepest fixable path with $\nabla P(v_{0}) \neq 0$ must
  be a simple path.
\begin{definition}
  Given a steepest fixable path $P$ in an instance $(G,v_{0}),$
  we define $\fixpath_{G}[v_{0},P] : V_{G} \to \rea \cup \{*\}$ to be
  the voltage assignment defined as follows
\begin{align*}
\fixpath_{G}[v_{0},P](x) =
\begin{cases}
v_{0}(\partial_{0} P) - \nabla P(v_{0}) \cdot \len_{G}(P[\partial_{0} P : x]) &
x \in \interior(P) \setminus T(v_{0}), \\
v_{0}(x) & \text{otherwise.} 
\end{cases}
\end{align*}
\end{definition}
We say that the vertices $x \in \interior(P)$ are fixed by the
operation $\fixpath[v_0,P].$ If we define
$v_1 = \fixpath_{G}[v_{0},P],$ where $P = (x_{0},\ldots,x_r)$ is the
steepest fixable path in $(G,v_{0}),$ then it is easy to argue that for every $i \in [r],$ we
have $\grad[v_1](x_{i-1},x_{i}) = \nabla P$ (see
Lemma~\ref{lem:gradSteepestPath}). The meta-algorithm \meta,
spelled out as Algorithm~\ref{fig:meta-algorithm}, entails repeatedly fixing steepest
fixable paths.  While it is possible to have multiple steepest fixable
paths, the result of fixing all of them does not depend on the order
in which they are fixed.
%
%
\begin{theorem}
\label{thm:basics:meta}
Given a well-posed instance $(G,v_{0})$, the meta-algorithm \meta, which repeatedly fixes steepest
fixable paths, produces the unique lex-minimizer extending $v_{0}$.
\end{theorem}
\begin{corollary}
  \label{cor:fix-path}
  Given a well-posed instance $(G,v_{0})$ such that
  $T(v_{0}) \neq V_{G},$ let $P$ be a steepest fixable path in
  $(G,v_{0}).$ Then, $(G,\fixpath[v_{0},P])$ is also a well-posed
  instance, and $\lex_{G}[\fixpath[v_{0},P]] = \lex_G[v_{0}].$
\end{corollary}

%




Since a lex-minimal element must be an inf-minimizer, we
also obtain the following corollary, that can also be proved using LP duality.
\begin{lemma}
\label{lem:basics:inf-duality}
  Suppose we have a well-posed instance $(G,v_0).$ Then, there exists
  a complete voltage assignment $v$ extending $v_0$ such that
  $\norm{\grad[v]}_{\infty} \le \alpha$, iff every terminal path $P$ in
  $(G,v_{0})$ satisfies $\nabla P(v_0) \le \alpha.$
\end{lemma}

\subsection{Stability}\label{ssec:stability}



The following theorem states that $\lex_G[v_{0}]$ is monotonic with respect to $v_{0}$ and it respects scaling and translation of $v_{0}$. 
\begin{theorem}
\label{thm:monotonicity}
Let $(G,v_{0})$ be a well-posed instance with $T := T(v_{0})$ as the set of terminals.  Then the following statements hold.
\begin{enumerate}[itemsep=1pt]
\item For any $c,d \in \rea$, $v_{1}$ a partial assignment with terminals $T(v_{1})=T$ and $v_{1}(t) = cv_{0}(t) + d$ for all $t \in T$. Then,    
$\lex_G[v_{1}](i)  = c\cdot \lex_G[v_{0}](i) + d$ for all $i \in V_G$.
\item  $v_{1}$ a partial assignment with terminals $T(v_{1}) = T.$  Suppose further that $v_{1} (t) \ge v_{0} (t)$ for all $t \in T.$ Then, 
$\lex_G[v_{1}](i)  \ge \lex_G[v_{0}](i)$  for all $i \in V_G$.
\end{enumerate}
\end{theorem}
As a corollary, the above theorem gives a nice stability property that lex-minimal elements satisfy.
\begin{corollary}
\label{cor:stability}
Given well-posed instances $(G,v_{0})$, $(G,v_{1})$ such that $T := T(v_{0}) = T(v_{1})$, let 
$\epsilon := \max_{t \in T} |v_{0}(t) - v_{1}(t)|$. 
Then 
$ |\lex_G[v_{0}](i) - \lex_G[v_{1}](i)| \leq \epsilon \text{ for all } i \in V_G.$
\end{corollary}


\subsection{Alternate Characterizations}\label{ssec:characterizations}
There are at least two other seemingly disparate definitions that are
equivalent to lex-minimal voltages.

\paragraph{$l_p$-norm Minimizers.}
As mentioned in the introduction, for a well-posed instance
$(G,v_{0})$ the lex-minimizer is also the limit of $l_{p}$
minimizers. This follows from existing results about the limit of
$l_{p}$-minimizers~(\cite{Egger1990}) in affine spaces, since
$\{\grad[v]\ |\ v \text{ is complete}, v \text{ extends } v_{0}\}$
forms an affine subspace of $\rea^m.$ Thus, we have the following theorem:
\begin{theorem}[Limit of $l_{p}$-minimizers, follows from~\cite{Egger1990}]
  For any $p \in (1,\infty),$ given a well-posed instance
$(G,v_{0})$ define $v_{p}$ to be the unique complete
  voltage assignment extending $v_{0}$ and minimizing
  $\norm{\grad[v]}_{p},$ \emph{i.e.}
  \[v_p = \argmin_{\substack{v \text{ is complete} \\ v \text{ extends
      } v_{0}}} \norm{\grad[v]}_p.\]
Then, $\lim_{p \to \infty} v_{p} = \lex_G[v_{0}].$
\end{theorem}

\paragraph{Max-Min Gradient Averaging.}
Consider a well-posed instance $(G,v_{0}),$ and a complete voltage
assignment $v$ extending $v_{0}.$ If $G$ is such that $\len(e) = 1$
for all $e \in E_{G},$ it is easy to see that
$\lex \defeq \lex_{G}[v_{0}]$ satisfies the following simple condition
for all $x \in V_{G} \setminus T(v_{0}),$
$$\lex(x) = \frac{1}{2} \left(\max_{(x,y) \in E_{G}} \lex(y) +
  \min_{(x,z) \in E_{G}}\lex(z)\right).$$
This condition should be contrasted to the optimality condition for
$l_{2}$-regularization on these instances, which gives for all
non-terminals $x,$ the optimal voltage $v$ satisfies
$v(x) = \frac{1}{\mathsf{deg}(x)}\sum_{y : (x,y) \in E_{G}} v(y).$ 

To prove the above claim, consider locally changing $\lex$ at $x$ and
observe that the gradients of edges not incident at $x$ remain
unchanged, and at least one of edges incident at $x$ will have a
strictly larger gradient, contradicting lex-minimality. For general
graphs, this condition of local optimality can still be characterized
by a simple \emph{max-min gradient averaging property} as described
below.
\begin{definition}[Max-Min Gradient Averaging]
  Given a well-posed instance $(G,v_{0}),$ and a complete voltage
  assignment $v$ extending $v_{0},$ we say that $v$ satisfies the
  \emph{max-min gradient averaging} property (w.r.t. $(G,v_{0})$) if
  for every $x \in V_G \setminus T(v_{0}),$ we have \vspace{-6pt}
  \[
\vspace{-6pt}
  \max_{y:(x,y) \in E_{G}} \grad[v](x,y) = - \min_{y:(x,y) \in E_{G}}
  \grad[v](x,y).
  \]
\end{definition}

As stated in the theorem below, $\lex_{G}[v_{0}]$ is the unique
assignment satisfying max-min gradient averaging
property. \cite{Smart} proved a variant of this statement for weighted graphs. For completeness, we
present a proof in the appendix.
\begin{theorem}
\label{thm:basics:max-min}
Given a well-posed instance $(G,v_{0}),$ $\lex_G[v_{0}]$ satisfies
max-min gradient averaging property. Moreover, it is the unique
complete voltage assignment extending $v_{0}$ that satisfies this
property w.r.t. $(G,v_{0}).$
\end{theorem}
An advantage of this characterization is that it can be verified
quickly. This is particularly useful for implementations for computing
the lex-minimizer.


\newcommand{\moddijkstra}{{\sc ModDijkstra}}
\newcommand{\steepestpath}{{\sc SteepestPath}}
\newcommand{\compinfmin}{{\sc CompInfMin}}
\newcommand{\complexmin}{{\sc CompLexMin}}
\newcommand{\compfastlexmin}{{\sc CompFastLexMin}}
\newcommand{\compvlow}{{\sc CompVLow}}
\newcommand{\compvhigh}{{\sc CompVHigh}}
\newcommand{\comphighpressgraph}{{\sc CompHighPressGraph}}
\newcommand{\fixabovepress}{{\sc FixAbovePress}}
\newcommand{\fixpathsabovepress}{{\sc FixPathsAbovePress}}
\newcommand{\compapproxlex}{{\sc CompApproxLex}}
\newcommand{\vertexsteepestpath}{{\sc VertexSteepestPath}}
\newcommand{\starsteepestpath}{{\sc StarSteepestPath}}

\newcommand{\parent}{\ensuremath{\mathsf{parent}}}
\newcommand{\LParent}{\ensuremath{\mathsf{LParent}}}
\newcommand{\HParent}{\ensuremath{\mathsf{HParent}}}
\newcommand{\temp}{\ensuremath{\mathsf{temp}}}
\newcommand{\vLow}{\ensuremath{\mathsf{vLow}}}
\newcommand{\vHigh}{\ensuremath{\mathsf{vHigh}}}
\newcommand{\treeRoot}{\ensuremath{\mathsf{root}}}
\newcommand{\pressure}{\ensuremath{\mathsf{pressure}}}
\newcommand{\sfd}{\ensuremath{\mathsf{d}}}

\newlength{\algtopspace}
\setlength{\algtopspace}{-4pt}
\newlength{\algpostcaptionspace}
\setlength{\algpostcaptionspace}{0pt}

\vspace{-7pt}
\section{Algorithms}\label{sec:algs}
We now sketch the ideas behind our algorithms
and give precise statements of our results. A full
description of all the algorithms is included in the appendix.

We define the \textbf{pressure} of a vertex to be the gradient of the steepest
  terminal path through it:
\[
\pressure[v_{0}](x) \defeq \max\{\nabla P(v_{0})\ |\ P \textrm{ is a 
  terminal path in $(G,v_{0})$ and } x \in P\}.
\]
Observe that in a graph with no terminal-terminal edges, a free
terminal path is a steepest fixable path iff its gradient is equal to
the highest pressure amongst all vertices. Moreover, vertices that lie
on steepest fixable paths are exactly the vertices with the highest
pressure. For a given $\alpha > 0,$ in order to identify vertices with
pressure exceeding $\alpha,$ we compute vectors $\vHigh[\alpha](x)$
and $\vLow[\alpha](x)$ defined as follows in terms of $\dis$, the
metric on $V$ induced by $\ell$:
\[
  \vLow[\alpha](x) = \min_{t \in T(v_{0})} \{ v_{0}(t) + \alpha\cdot \dis(x,t)\} \qquad 
  \vHigh[\alpha](x) = \max_{t \in T(v_{0})} \{ v_{0}(t) - \alpha\cdot \dis(t,x)\}.
\]

\subsection{Lex-minimization on Star Graphs}
We first consider the problem of computing the lex-minimizer on a star graph in which
  every vertex but the center is a terminal.
This special case is a subroutine in the general algorithm,
  and also motivates some of our techniques.

  Let $x$ be the center vertex, $T$ be the set of terminals, and all
  edges be of the form $(x,t)$ with $t \in T$.  The initial voltage
  assignment is given by $v : T \to \rea,$ and we abbreviate
  $\dis(x,t)$ by $\sfd(t) = \len(x,t).$ From
  Corollary~\ref{cor:fix-path} we know that we can determine the value
  of the lex minimizer at $x$ by finding a steepest fixable path. By
  definition, we need to find $t_{1},t_{2} \in T$ that maximize the
  gradient of the path from $t_{1}$ to $t_{2},$
  $\nabla(t_{1},t_{2}) \defeq \frac{v(t_{1}) - v(t_2)}{\sfd(t_{2}) +
    \sfd(t_{2})}.$
  As observed above, this is equivalent to finding a terminal with the
  highest pressure.  We now present a simple randomized algorithm for
  this problem that runs in expected linear time.

  Given a terminal $t_{1}$, we can compute its pressure $\alpha$ along
  with the terminal $t_{2}$ such that $|\nabla(t_{1},t_{2})| = \alpha$
  in time $O (\sizeof{T})$ by scanning over the terminals in $T$.
  Consider doing this for a random terminal $t_{1}$.  We will show
  that in linear time one can then find the subset of terminals
  $T' \subset T$ whose pressure is greater than $\alpha$.  Assuming
  this, we complete the analysis of the algorithm.  If
  $T^\prime = \emptyset,$ $t_{1}$ is a vertex with highest
  pressure. Hence the path from $t_{1}$ to $t_{2}$ is a steepest
  fixable path, and we return $(t_{1},t_{2}).$ If
  $T^\prime \neq \emptyset,$ the terminal with the highest pressure
  must be in $T^{\prime},$ and we recurse by picking a new random
  $t_{1} \in T^{\prime}.$ As the size of $T'$ will halve in
  expectation at each iteration, the expected time of the algorithm on
  the star is $O(|T|)$.

  To determine which terminals have pressure exceeding $\alpha$, we
  observe that the condition
  $\exists t_{2}: \alpha < \nabla(t_{1},t_{2}) =
  \frac{v(t_{1})-v(t_{2})}{\sfd(t_1) + \sfd(t_{2})},$
  is equivalent to
  $\exists t_{2}: v(t_{2}) + \alpha \sfd(t_2) < v(t_{1}) - \alpha
  \sfd(t_{1}).$
  This, in turn, is equivalent to
  $\vLow[\alpha](x) < v(t_{1}) - \alpha \sfd(t_{1}).$ We can compute
  $\vLow[\alpha](x)$ in deterministic $O(|T|)$ time. Similarly, we can
  check if $\exists t_{2}: \alpha < \nabla(t_{2},t_{1})$ by checking
  if $\vHigh[\alpha](x) > v_{t_1} + \alpha\sfd(t_{1}).$ Thus, in
  linear time, we can compute the set $T'$ of terminals with pressure
  exceeding $\alpha$.  The above algorithm is described in
  Algorithm~\ref{alg:star-steepest-path}.
\begin{theorem}\label{thm:star}
Given a set of terminals $T,$ initial voltages $v:T \to
\rea,$ and distances $\sfd : T \to \rea_{+},$
\starsteepestpath$(T,v,\sfd)$ returns $(t_1,t_2)$ maximizing
$\frac{v(t_{1}) - v(t_{2})}{\sfd(t_1) + \sfd(t_{2})},$ and runs in
expected time $O(|T|).$
\end{theorem}

\subsection{Lex-minimization on General Graphs}

Theorem~\ref{thm:basics:meta}, tells us that \meta\, will compute
lex-minimizers given an algorithm for finding a steepest fixable path
in $(G,v_{0}).$ Recall that finding a steepest fixable path is
equivalent to finding a path with gradient equal to the highest
pressure amongst all vertices. In this section, we show how to do this
in expected time $O (m + n \log n)$.

We describe an algorithm \vertexsteepestpath\, that
  finds a terminal path $P$ through any vertex $x$
  such that
  $\nabla P(v_{0}) = \pressure[v_{0}](x)$ in expected $O(m + n\log n)$
time. 
Using Dijkstra's algorithm,  we compute
   $\dis(x,t)$ for all $t \in T.$ 
If $x \in T(v_{0}),$ then there must be a terminal path $P$
  that starts at $x$ that has
  $\nabla P(v_{0}) = \pressure[v_{0}](x).$ 
To compute such a $P$ we examine 
  all $t \in T(v_{0})$ in $O(|T|)$ time to find the $t$ that maximizes
  $|\nabla(x,t)| = \frac{|v(x)-v(t)|}{\dis(x,t)},$ and then return a shortest
path between $x$ and that $t.$

If $x \notin T(v_{0}),$ then the steepest path through $x$
between terminals $t_{1}$ and $t_{2}$ must consist of shortest paths
  between $x$ and $t_{1}$ and between $x$ and $t_{2}$.
Thus, we can reduce the problem to that of finding the steepest path
  in a star graph where $x$ is the only non-terminal and is connected to 
  each terminal $t$ by an edge of length $\dis(x,t).$ 
By Theorem \ref{thm:star}, we can find this steepest path in  $O(|T|)$ expected time.
The above algorithm is formally described as
  Algorithm~\ref{alg:vertex-steepest-path}.

\begin{theorem}\label{thm:steepest}
  Given a well-posed instance $(G,v_{0}),$ and a vertex $x \in V_{G},$
  \vertexsteepestpath$(G,v_{0},x)$ returns a terminal path $P$
  through $x$ such that $\nabla P(v_{0}) = \pressure[v_{0}](x),$ in
  $O(m+n\log n)$ expected time.
\end{theorem}

As in the algorithm for the star graph, we need to
identify the vertices whose pressure exceeds a given $\alpha$.
For a fixed $\alpha,$ we can compute $\vLow[\alpha](x)$ and
$\vHigh[\alpha](x)$ for all $x \in V_{G}$ using a simple modification
of Dijkstra's algorithm in $O(m + n\log n)$ time. We describe the
algorithms \compvhigh, {\compvlow} for these tasks in
Algorithms~\ref{alg:vlow} and~\ref{alg:vhigh}. The following lemma
encapsulates the usefulness of $\vLow$ and $\vHigh.$
\begin{lemma}
\label{lem:vlowvhigh-pressure}
For every $x \in V_{G},$ $\pressure[v_{0}](x) > \alpha$
iff $\vHigh[\alpha](x) >\vLow[\alpha](x).$
\end{lemma}

It immediately follows that the algorithm
\comphighpressgraph$(G,v_{0},\alpha)$ described in
Algorithm~\ref{alg:comp-high-press-graph} computes the vertex induced
subgraph on the vertex set $\{x \in V_{G}| \ \pressure[v_{0}] (x) >
\alpha\}.$

We can combine these algorithms into an algorithm \steepestpath\, that finds the steepest fixable
path in $(G,v_{0})$ in $O(m + n\log n)$ expected time. We may assume that there are
no terminal-terminal edges in $G.$ We sample an edge $(x_{1},x_{2})$
uniformly at random from $E_{G}$, and a terminal $x_{3}$ uniformly at
random from $V_{G}.$ 
For $i=1,2,3,$ we compute the steepest terminal
  path $P_{i}$ containing $x_{i}.$
By Theorem~\ref{thm:steepest}, this
  can be done in $O(m+n\log n)$ expected time. 
Let $\alpha$ be the largest gradient $\max_{i} \nabla P_{i}.$ 
As mentioned above, we
can identify $G^{\prime},$ the induced subgraph on vertices $x$ with
 pressure exceeding $\alpha$, in $O(m + n \log n)$ time. 
If $G^\prime$ is empty, we know that the path $P_{i}$ with largest gradient is
  a steepest fixable path. 
If not, a steepest
  fixable path in $(G,v_{0})$ must be in $G^{\prime},$ and hence we can
  recurse on $G^{\prime}.$ 
Since we picked a uniformly random edge, and
  a uniformly random vertex, the expected size of $G^{\prime}$ is at most half
  that of $G$. Thus, we obtain an expected running time
of $O(m + n \log n).$
This algorithm is described in detail in Algorithm~\ref{alg:steepest-path}.

\begin{theorem}
\label{thm:steepestpathalg}
  Given a well-posed instance $(G,v_{0})$ with
  $E_{G} \cap (T(v_{0}) \times T(v_{0})) = \emptyset$,
  \steepestpath$(G,v_{0})$ returns a steepest fixable path in
  $(G,v_{0}),$ and runs in $O(m+n\log n)$ expected time. 
\end{theorem}
By using \steepestpath\,  in \meta, we get the \complexmin,
  shown in Algorithm~\ref{fig:meta-algorithm}. From Theorem~\ref{thm:basics:meta}
  and Theorem~\ref{thm:steepestpathalg}, we immediately get the following corollary. 
\begin{corollary}
 Given a well-posed instance $(G,v_{0})$ as input, algorithm \complexmin\, 
  computes a lex-minimizing assignment that extends $v_0$ in 
 $O(n (m + n \log n) )$ expected time.
\end{corollary}

\subsection{Linear-time Algorithm for Inf-minimization}\label{sec:infAlg}
Given the algorithms in the previous section, it is straightforward to
construct an infinity minimizer. Let $\alpha^{\star}$ be the gradient
of the steepest terminal path. From 
Lemma~\ref{lem:basics:inf-duality}, we know that the norm of the inf
minimizer is $\alpha^{\star}$. Considering all trivial terminal paths
(terminal-terminal edges), and using \steepestpath, we can compute
$\alpha^{\star}$ in randomized $O(m + n \log n)$ time. It is well
known (\cite{McShane,Whitney}) that $v_{1} \defeq \vLow[\alpha^{\star}]$
and $v_{2} \defeq \vHigh[\alpha^{\star}]$ are inf-minimizers. It is
also known that $\frac{1}{2} (v_{1} + v_{2})$ is the inf-minimizer
that minimizes the maximum $\ell_\infty$-norm distance to all inf-minimizers. In the case of path graphs, this was observed by \cite{GaffneyOptLip} and independently by \cite{WinogradOptLip}.
For
completeness, the algorithm is presented as
Algorithm~\ref{alg:comp-inf-min}, and we have the following result.
\begin{theorem}
\label{thm:lineartimeinfmin}
Given a well-posed instance $(G,v_{0}),$ \compinfmin$(G,v_{0})$
returns a complete voltage assignment $v$ for $G$ extending $v_{0}$ that
minimizes $\norm{\grad[v]}_{\infty},$ and runs in randomized $O(m + n
\log n)$ time.
\end{theorem}

\subsection{Faster Algorithms for Lex-minimization}\label{sec:fastAlg}
The lex-minimizer
has additional structure that allows one to compute it by more efficient
algorithms.  
One observation that leads to a faster implementation is that
  fixing a steepest fixable path does not increase the pressure at vertices,
  provided that one appropriately ignores terminal-terminal edges.
Thus, if $G^{(\alpha)}$ is a subgraph
that we identified with pressure greater than $\alpha,$ we can
iteratively fix all steepest fixable paths $P$ in $G^{(\alpha)}$ with
$\nabla P > \alpha.$ Another simple observation is that if
$G^{(\alpha)}$ is disconnected, we can simply recurse on each of the
connected components. A complete description of an the algorithm
{\compfastlexmin} based on these idea is given in
Algorithm~\ref{alg:comp-fast-lex-min}. The algorithm provably computes
$\lex_{G}(v_{0}),$ and it is possible to implement it so that the
space requirement is only $O(m+n).$ Although, we are unable to prove
theoretical bounds on the running time that are better than
$O(n(m+n\log n))$,
it runs extremely quickly in practice.
We used it to perform the experiments in this paper.
For random regular graphs and Delaunay graphs, with $n = 0.5 \times 10^6$ vertices and around 2 million edges $m \sim 1.5-2 \times 10^6$, it takes a couple of minutes on a 2009 MacBook Pro. Similar times are observed for other model graphs of this size such as random
  regular graphs and real world networks. An implementation of this algorithm may be found at \texttt{https://github.com/danspielman/YINSlex}.




\section{Directed Graphs}
\label{sec:directed}
Our definitions and algorithms, including those for regularization,
extend to directed graphs with only small modifications.  We view
directed edges as diodes and only consider potential differences in
the direction of the edge.  For a complete voltage assignment $v$ on
the vertices of a directed graph $G$, we define the directed gradient
on $(x,y)$ due to $v$ to be
$\grad^{+}_{G}[v](x,y) \defeq \max \left \{
  \frac{v(x)-v(y)}{\len(x,y)}, 0 \right \}. $
Given a partially-labelled directed graph $(G,v_0)$, we say that a a complete
voltage assignment $v$ is a \textbf{lex-minimizer} if it extends
$v_{0}$ and for other complete voltage assignment $v'$ that extends
$v_{0}$ we have $\grad^{+}_{G} [v] \lexless \grad^{+}_{G} [v'].$ We
say that a partially-labelled directed graph $(G,v_0)$ is a \textbf{well-posed} directed instance
if every free vertex appears in a directed path between two terminals.

The main difference between the directed and undirected cases is that
the directed lex-minimizer is not necessarily unique. To maintain
clarity of exposition, we chose to focus on undirected graphs so
far. For directed graphs, we have the following corresponding
structural results.  
\begin{theorem}
\label{thm:basics:directed-existence}
Given a well-posed instance $(G,v_{0})$ on a directed graph $G$, there
exists a lex-minimizer, and the set of all lex-minimizers is a convex
set. Moreover, for every two lex-minimizers $v$ and $v'$, we have
$\grad^{+}_{G} [v]=\grad^{+}_{G} [v']$.
\end{theorem}
However, note that in the case of directed graphs, the lex-minimizer
need not be unique. We still have a weaker version of
Theorem~\ref{thm:basics:meta} for directed graphs.
\begin{theorem}
\label{thm:basics:directed-meta}
Given a well-posed instance $(G,v_{0})$ on a directed graph $G$, let
$v_{1}$ be the partial voltage assignment extending $v_{0}$ obtained
by repeatedly fixing steepest fixable (directed) paths $P$ with
$\nabla P > 0.$ Then, any lex-minimizer of $(G,v_{0})$ must extend
$v_{1}.$ Moreover, for every edge
$e \in E_{G} \setminus (T(V_{1}) \times T(V_{1})),$ any lex-minimizer
$v$ of $(G,v_{0})$ must satisfy $\grad^{+}[v](e) = 0.$
\end{theorem}

When the value of the lex-minimizer at a vertex is not uniquely
determined, it is constrained to an interval. In our experiments, we
pick the convention that when the voltage at a vertex is constrained
to an interval $({-\infty},a]$ or $[a,\infty)$, we assign $a$ to the
terminal. When it is constrained to a finite interval, we assign a
voltage closest to the median of the original
voltages. 

\section{Experiments on WebSpam}
\label{sec:experiments}
We demonstrate the performance of our lex-minimization algorithms on directed graphs
  by using them to detect spam webpages as in  \cite{Zhou2007}.
We use the dataset \texttt{webspam-uk2006-2.0} described in \cite{Castillo2006}. 
This collection includes 11,402 hosts, out of which 7,473 (65.5~\%) are labeled, either as \texttt{spam} or \texttt{normal}. 
Each host corresponds to the collection of web pages it serves.
Of the hosts, 1924 are labeled spam (25.7~\% of all labels). 
We consider the problem of flagging some hosts as spam, given only a small fraction of the labels for training.
We assign a value of $1$ to the spam hosts, and a value of $0$ to the normal ones.
We then compute a lex minimizer and examine the effect of flagging as spam all hosts with a value greater than some threshold.

Following \cite{Zhou2007}, we create edges between hosts with lengths equal to the reciprocal of 
  the number of links from one to the other. 
We run our experiments only on the largest strongly connected component of the graph, which contains 7945 hosts of which 5552 are labeled.
16~\% of the nodes in this subgraph are labeled \texttt{spam}.
To create training and test data, for a given value $p$, we select a random subset of $p$~\% of the \texttt{spam} labels and a random subset of 
  $p$~\% of the \texttt{normal} labels to use for training. 
The remaining labels are used for testing. We report results for $p = 5$ and $p = 20$.

Again following \cite{Zhou2007}, we plot the precision and recall of different choices of threshold for flagging pages as spam.
\emph{Recall} is the fraction of spam pages our algorithm flags as spam, and \emph{precision}
  is the fraction of pages our algorithm flags as spam that actually are spam.
Amongst the algorithms studied by \cite{Zhou2007}, the top performer was
  their algorithm based on sampling according to a random-walk that
  follows in-links from other hosts.   
We compare their algorithm with the classification we get by 
  directing edges in the opposite directions of links.
This has the effect that a link to a spam host is evidence of spamminess,
  and a link from a normal host is evidence of normality.
  
Results are shown in Figure~\ref{fig:experiment_image}.
  While we are not able to reliably flag all spam hosts, we see that in the range of {10-50~\%} recall,
  we are able to flag spam with precision above 82~\%.
  We see that the performance of directed lex-minimization does not degrade rapidly when from the ``large training set" regime of  $p = 20$, to the ``small training set" regime of $p = 5$.
\begin{figure}[ht]
\vspace{-0in}
    \centering
    \includegraphics[scale=0.4]{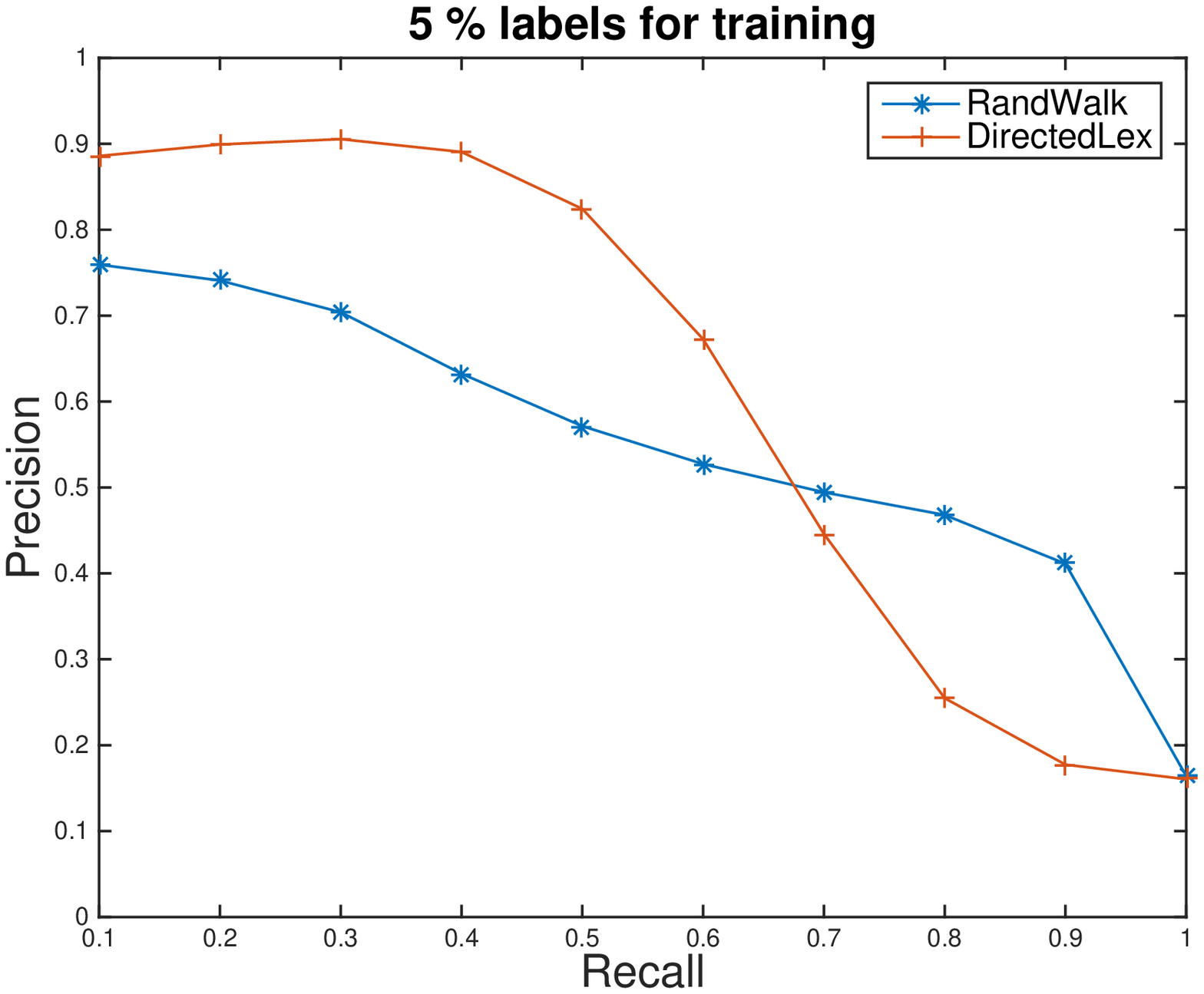}
\includegraphics[scale=0.4]{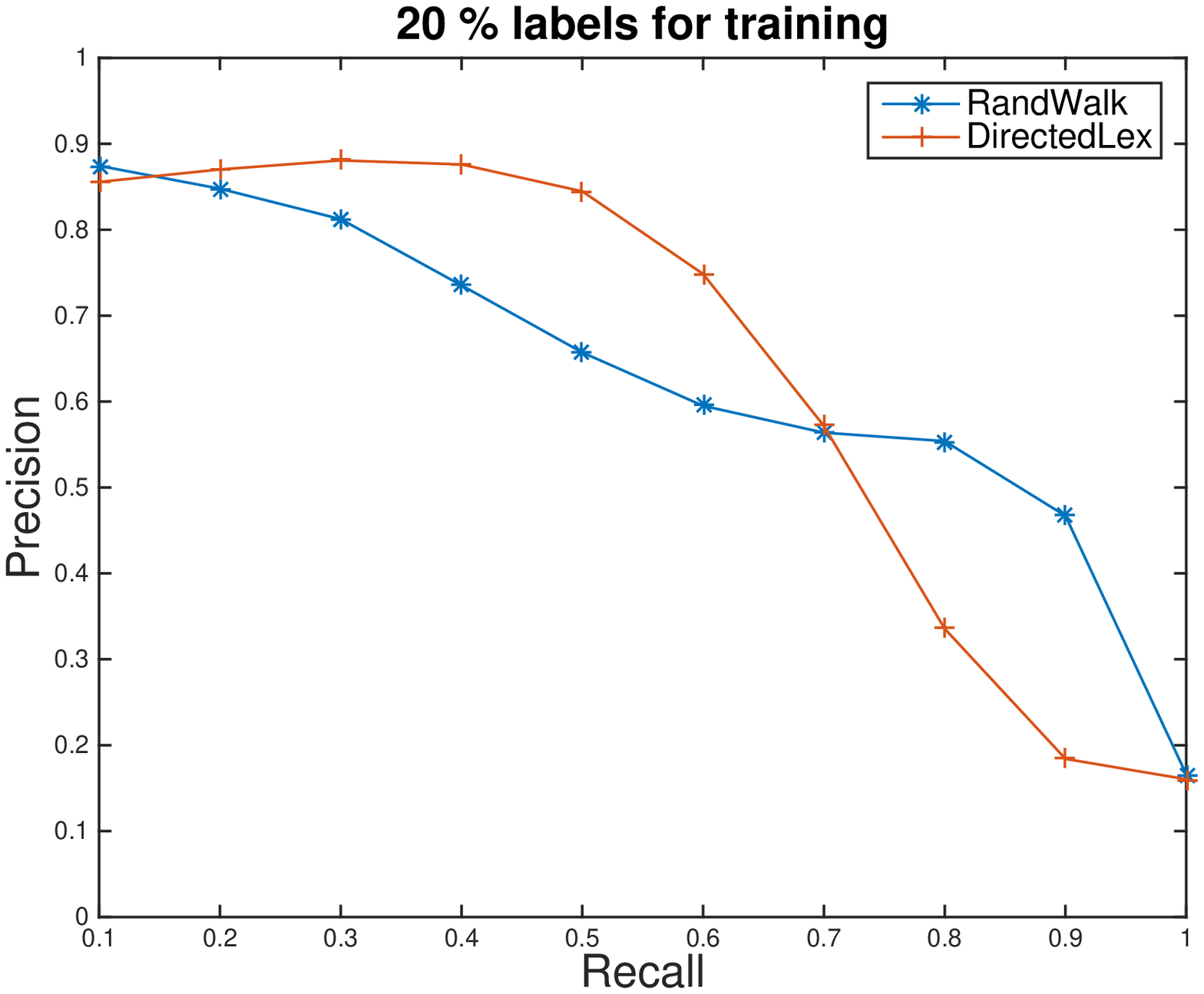}
\vspace{-0in}
\caption{Recall and precision in the web spam classification experiment.
  Each data point shown was computed as an average over 100 runs.
  The largest standard deviation of the mean precision across the plotted recall values was less than 1.3~\%.
  The algorithm of \cite{Zhou2007} appears as \textsc{RandWalk}.
  Our directed lex-minimization algorithm appears as \textsc{DirectedLex}. }
    \label{fig:experiment_image}
\end{figure}

For comparison, in Appendix~\ref{sec:all-directions-comparison}, we show the performance of our algorithm and that of \cite{Zhou2007} both with link directions reversed, as well as the performance of undirected lex-minimization and Laplacian inference, all of which are significantly worse.

\section{$l_{0}$-Regularization of Vertex Values}

We now explain how we can accommodate noise in both the given voltages and in the given lengths of edges.
We can find the minimum number of labels to ignore, or the minimum increase in edges lengths
  needed so that there exists an extension whose gradients have $l_{\infty}$-norm lower than a given target.
After determining which labels to ignore or the needed increment in edge lengths, we recommend
  computing a lex minimizer.

The algorithms we present in this section are essentially the same for directed and undirected graphs.



\subsection{$l_{0}$-Vertex Regularization for Inf-minimization}
\label{sec:l0reg}

The $l_{0}$-regularization of vertex labels can be viewed as a problem of outlier removal:
  the vector we compute is allowed to disagree with $v_{0}$ on up to $k$ terminals.
Given a voltage assignment $v$ and a subset $T \subset V$ of the vertices, by $v(T)$ we mean the vector obtained by restricting $v$ to $T$. We define the \emph{$l_{0}$-Vertex Regularization for $l_{\infty}$} problem to be
\vspace{-6pt}
\begin{equation}\label{eq:outlier-min-fixed-k}
 \min_{v \in \Reals{n}} \,\,
		 \norm{\grad_{G}[v]}_{\infty}
\qquad 
		 \text{subject to } \norm{v (T) -v_{0} (T)}_{0} \leq k,
\end{equation}
where $v (T)$ is the vector of values of $v$ on the terminals $T$.

In Appendix~\ref{sec:l0-proofs}, we describe an approximation algorithm \approxoutlieralg \ that approximately solves program~\eqref{eq:outlier-min-fixed-k}. The precise statement we prove in Appendix~\ref{sec:l0-proofs} is given in the following theorem.
\begin{theorem}[Approximate $l_{0}$-vertex regularization]
\label{thm:l0-approx}
The algorithm \approxoutlieralg \ takes a positive integer $k$ and a partially-labeled graph $(G,v_{0})$, and outputs an assignment $v$ with $\norm{v (T) -v_{0} (T)}_{0} \leq 2 k$, and $ \norm{\grad_{G}[v]}_{\infty} \leq \alpha^*$, where $\alpha^*$ is the optimum value of program~\eqref{eq:outlier-min-fixed-k}. The algorithm runs in time $O(k (m + n \log n))$.
\end{theorem}
In Appendix~\ref{sec:l0-proofs}, we also describe an algorithm \outlieralg \ that exactly solves program~\eqref{eq:outlier-min-fixed-k} in polynomial time, and we prove its correctness.
\begin{theorem}[Exact $l_{0}$-vertex regularization]
\label{thm:l0-exact}
The algorithm \outlieralg \ takes a positive integer $k$ and a partially-labeled graph $(G,v_{0})$ solves program~\eqref{eq:outlier-min-fixed-k} exactly. The algorithm runs in polynomial time.
\end{theorem}
We give a proof of Theorem~\ref{thm:l0-exact} in Appendix~\ref{sec:l0-proofs}. To do this, we reduce the program~\eqref{eq:outlier-min-fixed-k} to the problem of minimizing the required $l_0$-budget needed to achieve a fixed gradient $\alpha$ using a binary search over a set of $O(n^2)$ gradients. This latter problem we reduce in polynomial time to Minimum Vertex Cover (VC) on a transitively closed, directed acyclic graph (a TC-DAG). VC on a TC-DAG can be solved exactly in polynomial time by a reduction to the Maximum Bipartite Matching Problem (\cite{Fulkerson56noteon}). The problem was phrased by Fulkerson as one of finding a maximum antichain of a finite poset. Any transitively closed DAG corresponds directly to the comparability graph of a poset. A maximum antichain of a poset is a maximum independent set of a the comparability graph of the poset, and hence its complement is a minimum vertex cover of the comparability graph. We refer to the algorithm developed by Fulkerson as \tcdagvc.
\begin{theorem}
\label{thm:minvc-tcdag}
The algorithm \tcdagvc \ computes a minimum vertex cover for any transitively closed DAG $G$ in polynomial time.
\end{theorem}

\subsection{Hardness of $l_{0}$ regularization for $l_{2}$}
\label{sec:l2hardness}

The result that $l_{0}$-regularized inf-minimization can be solved exactly in polynomial time is surprising, especially because the analogous problem for 2-Laplacian minimization turns out to be NP-Hard.

We define the 
    the \emph{$l_{0}$ vertex regularization for $l_{2}$} for a partially-labeled graph $(G, v_{0})$ and an integer $k$ by
\[
\min_{v \in \rea^{n} : \norm{v (T) -v_{0} (T)}_{0} \leq k} v^{T} L v,
\]
where $L$ is the Laplacian of $G$.
\begin{theorem}
\label{thm:l0_l2_nphard}
$l_{0}$ vertex regularization for $l_{2}$ is NP-Hard.
\end{theorem}
In Appendix~\ref{sec:l2_hardness_proof} we prove Theorem~\ref{thm:l0_l2_nphard} by giving a polynomial time (Karp) reduction from the NP-Hard minimum bisection problem to $l_{0}$ vertex regularization for $l_{2}$.

\section{$l_{1}$-Edge and Vertex Regularization of Inf-minimizers}
\label{sec:l1_reg}

Consider a partially-labeled graph $(G,v_{0})$ and an $\alpha > 0$. The set of voltage assignments given by
$$\setof{ v : v \text{ extends } v_{0} \text{ and } \norm{\grad_{G}[v]}_{\infty } \leq \alpha }$$ is convex. Going further, let us consider the edge lengths in a graph to be specified by a vector $\len \in \Reals{E}$. Now the set of voltages $v$ and and lengths $\len$ which achieve $\| \grad_{G(\len)}[v] \|_{\infty} \leq \alpha$ is jointly convex in $v$ and $\len$. To see this, observe that
\begin{equation}
\label{eq:maxgradLP}
\| \grad_{G(\len)}[v] \|_{\infty} \leq \alpha \Leftrightarrow \forall (u,v) \in E :  \, {- \alpha \len(u,v)}\leq v(u) - v(v) \leq \alpha \len(u,v).
\end{equation}
Furthermore, the condition ``$v$ extends $v_{0}$'' is a linear constraint on $v$, which we express as $v (T) = v_{0} (T)$.
From the above, it is clear that the gradient condition corresponds to a convex set, as it is an intersection of half-spaces. These half-spaces are given by $O(m)$ linear inequalities.  We can leverage this to phrase many  regularized variants of inf-minimization as convex programs, and in some cases linear programs.

\renewcommand\dd{s}

For example, we may consider a variant of inf-minimization combined with an $l_1$-budget for changing lengths of edges and values on terminals. Given a parameter $\gamma > 0$ which specifies the relative cost of regularizing terminals to regularizing edges, the problem is as follows
\begin{equation}\label{eq:edge-l1-min}
 \argmin_{v \in \Reals{n}, \dd \in \Reals{m}, \dd \geq 0} \, 
		\norm{  \dd }_1 +  \gamma \norm{ v (T) - v_{0} (T) }_1 \\
\qquad  \text{subject to } 
		 \norm{\grad_{G(\ell + \dd)}[v]}_{\infty} \leq \alpha.
\end{equation}
From our observation~\eqref{eq:maxgradLP}, it follows that problem~\eqref{eq:edge-l1-min} may be expressed as a linear program with $O(n)$ variables and $O(m)$ constraints. 
We can use ideas from \cite{Daitch} to solve the resulting linear program in time $\widetilde{O}(m^{1.5})$
  by an interior point method with a special purpose linear equation solver.
The reason is that the linear equations the IPM must solve at each iteration
  may be reduced to linear equations in symmetric, diagonally dominant matrices,
  and these may be solved in nearly-linear time (\cite{cohen2014solving}).





\renewcommand\dd{\boldsymbol{\mathit{d}}}

\paragraph{Conclusion.}
We propose the use of inf and lex minimizers for
regression on graphs. We present simple algorithms for computing them
that are provably fast and correct, and can also be implemented 
efficiently. We also present a framework and polynomial time algorithms for
regularization in this setting. The initial experiments reported in
the paper indicate that these algorithms give pretty good results on
real and synthetic datasets. The results seem to compare quite
favorably to other algorithms, particularly in the regime of tiny
labeled sets. We are testing these algorithms on several other graph
learning questions, and plan to report on them in a forthcoming
experimental paper. We believe that inf and lex minimizers, and the
associated ideas presented in the paper, should be useful
primitives that can be profitably combined with other approaches
to learning on graphs.

\section*{Acknowledgements}
We thank anonymous reviewers for helpful comments. We thank Santosh Vempala and Bartosz Walczak for pointing out that it was already known how to compute a minimum vertex cover of a transitively closed DAG in polynomial time.

\bibliographystyle{plainnat}
\bibliography{references}

\begin{thebibliography}{37}
\providecommand{\natexlab}[1]{#1}
\providecommand{\url}[1]{\texttt{#1}}
\expandafter\ifx\csname urlstyle\endcsname\relax
  \providecommand{\doi}[1]{doi: #1}\else
  \providecommand{\doi}{doi: \begingroup \urlstyle{rm}\Url}\fi

\bibitem[Alamgir and Luxburg(2011)]{Ulrike_p_resist}
Morteza Alamgir and Ulrike~V. Luxburg.
\newblock Phase transition in the family of p-resistances.
\newblock In \emph{Advances in Neural Information Processing Systems 24}, pages
  379--387. 2011.
\newblock URL \url{http://books.nips.cc/papers/files/nips24/NIPS2011_0278.pdf}.

\bibitem[Aronsson(1967)]{Aronsson}
Gunnar Aronsson.
\newblock Extension of functions satisfying lipschitz conditions.
\newblock \emph{Arkiv för Matematik}, 6\penalty0 (6):\penalty0 551--561, 1967.
\newblock ISSN 0004-2080.
\newblock \doi{10.1007/BF02591928}.
\newblock URL \url{http://dx.doi.org/10.1007/BF02591928}.

\bibitem[Aronsson et~al.(2004)Aronsson, Crandall, and Juutinen]{Aronsson_atour}
Gunnar Aronsson, Michael~G. Crandall, and Petri Juutinen.
\newblock A tour of the theory of absolutely minimizing functions.
\newblock \emph{Bull. Amer. Math. Soc. (N.S.)}, 41\penalty0 (4):\penalty0
  439--505, 2004.
\newblock ISSN 0273-0979.
\newblock \doi{10.1090/S0273-0979-04-01035-3}.
\newblock URL \url{http://dx.doi.org/10.1090/S0273-0979-04-01035-3}.

\bibitem[Barles and Busca(2001)]{Barles01}
Guy Barles and J\'{e}r\^{o}me Busca.
\newblock Existence and comparison results for fully nonlinear degenerate
  elliptic equations without zeroth-order term.
\newblock \emph{Comm. Partial Differential Equations}, 26:\penalty0 2323--2337,
  2001.

\bibitem[Belkin et~al.(2004)Belkin, Matveeva, and Niyogi]{BelkinRegression}
Mikhail Belkin, Irina Matveeva, and Partha Niyogi.
\newblock Regularization and semi-supervised learning on large graphs.
\newblock In \emph{Learning Theory}, volume 3120 of \emph{Lecture Notes in
  Computer Science}, pages 624--638. Springer Berlin Heidelberg, 2004.
\newblock ISBN 978-3-540-22282-8.
\newblock \doi{10.1007/978-3-540-27819-1_43}.
\newblock URL \url{http://dx.doi.org/10.1007/978-3-540-27819-1_43}.

\bibitem[Bridle and Zhu(2013)]{NBridle}
Nick Bridle and Xiaojin Zhu.
\newblock $p$-voltages: {Laplacian} regularization for semi-supervised learning
  on high-dimensional data.
\newblock In \emph{Eleventh Workshop on Mining and Learning with Graphs
  (MLG2013)}, 2013.

\bibitem[Castillo et~al.(2006)Castillo, Donato, Becchetti, Boldi, Leonardi,
  Santini, and Vigna]{Castillo2006}
Carlos Castillo, Debora Donato, Luca Becchetti, Paolo Boldi, Stefano Leonardi,
  Massimo Santini, and Sebastiano Vigna.
\newblock A reference collection for web spam.
\newblock \emph{SIGIR Forum}, 40\penalty0 (2):\penalty0 11--24, December 2006.
\newblock ISSN 0163-5840.
\newblock \doi{10.1145/1189702.1189703}.
\newblock URL \url{http://doi.acm.org/10.1145/1189702.1189703}.

\bibitem[Chapelle et~al.(2010)Chapelle, Schlkopf, and Zien]{ChaSchZie06}
Olivier Chapelle, Bernhard Schlkopf, and Alexander Zien.
\newblock \emph{Semi-Supervised Learning}.
\newblock The MIT Press, 1st edition, 2010.
\newblock ISBN 0262514125, 9780262514125.

\bibitem[Cohen et~al.(2014)Cohen, Kyng, Miller, Pachocki, Peng, Rao, and
  Xu]{cohen2014solving}
Michael~B Cohen, Rasmus Kyng, Gary~L Miller, Jakub~W Pachocki, Richard Peng,
  Anup~B Rao, and Shen~Chen Xu.
\newblock Solving {SDD} linear systems in nearly $m \log^{1/2} n$ time.
\newblock In \emph{Proceedings of the 46th Annual ACM Symposium on Theory of
  Computing}, pages 343--352. ACM, 2014.

\bibitem[Crandall et~al.(2001)Crandall, Evans, and Gariepy]{Crandall}
M.G. Crandall, L.C. Evans, and R.F. Gariepy.
\newblock Optimal lipschitz extensions and the infinity laplacian.
\newblock \emph{Calculus of Variations and Partial Differential Equations},
  13\penalty0 (2):\penalty0 123--139, 2001.
\newblock ISSN 0944-2669.
\newblock \doi{10.1007/s005260000065}.
\newblock URL \url{http://dx.doi.org/10.1007/s005260000065}.

\bibitem[Daitch and Spielman(2008)]{Daitch}
Samuel~I. Daitch and Daniel~A. Spielman.
\newblock Faster approximate lossy generalized flow via interior point
  algorithms.
\newblock In \emph{Proceedings of the Fortieth Annual ACM Symposium on Theory
  of Computing}, STOC '08, pages 451--460, New York, NY, USA, 2008. ACM.
\newblock ISBN 978-1-60558-047-0.
\newblock \doi{10.1145/1374376.1374441}.
\newblock URL \url{http://doi.acm.org/10.1145/1374376.1374441}.

\bibitem[Egger and Huotari(1990)]{Egger1990}
Alan Egger and Robert Huotari.
\newblock Rate of convergence of the discrete polya algorithm.
\newblock \emph{Journal of Approximation Theory}, 60\penalty0 (1):\penalty0 24
  -- 30, 1990.
\newblock ISSN 0021-9045.
\newblock \doi{http://dx.doi.org/10.1016/0021-9045(90)90070-7}.
\newblock URL
  \url{http://www.sciencedirect.com/science/article/pii/0021904590900707}.

\bibitem[Fefferman(2009{\natexlab{a}})]{Fefferman}
Charles Fefferman.
\newblock Whitney's extension problems and interpolation of data.
\newblock \emph{Bull. Amer. Math. Soc. (N.S.)}, 46\penalty0 (2):\penalty0
  207--220, 2009{\natexlab{a}}.
\newblock ISSN 0273-0979.
\newblock \doi{10.1090/S0273-0979-08-01240-8}.
\newblock URL \url{http://dx.doi.org/10.1090/S0273-0979-08-01240-8}.

\bibitem[Fefferman(2009{\natexlab{b}})]{Fefferman3}
Charles Fefferman.
\newblock Fitting a [image] -smooth function to data, iii.
\newblock \emph{Annals of Mathematics}, 170\penalty0 (1):\penalty0 pp.
  427--441, 2009{\natexlab{b}}.
\newblock ISSN 0003486X.
\newblock URL \url{http://www.jstor.org/stable/40345469}.

\bibitem[Fefferman and Klartag(2009)]{Fefferman1}
Charles Fefferman and Bo'az Klartag.
\newblock Fitting a cm -smooth function to data i.
\newblock \emph{Annals of Mathematics}, 169\penalty0 (1):\penalty0 pp.
  315--346, 2009.
\newblock ISSN 0003486X.
\newblock URL \url{http://www.jstor.org/stable/40345445}.

\bibitem[Fulkerson(1956)]{Fulkerson56noteon}
D.~R. Fulkerson.
\newblock Note on dilworths decomposition theorem for partially ordered sets.
\newblock \emph{Proc. Amer. Math. Soc}, 1956.

\bibitem[Gaffney and Powell(1976)]{GaffneyOptLip}
P.W. Gaffney and M.J.D. Powell.
\newblock Optimal interpolation.
\newblock In \emph{Numerical Analysis}, volume 506 of \emph{Lecture Notes in
  Mathematics}, pages 90--99. Springer Berlin Heidelberg, 1976.
\newblock ISBN 978-3-540-07610-0.
\newblock \doi{10.1007/BFb0080117}.
\newblock URL \url{http://dx.doi.org/10.1007/BFb0080117}.

\bibitem[Gottlieb et~al.(2013)Gottlieb, Kontorovich, and
  Krauthgamer]{KrauthgamerClassification}
L.-A. Gottlieb, A.~Kontorovich, and R.~Krauthgamer.
\newblock Efficient classification for metric data.
\newblock \emph{CoRR}, abs/1306.2547, 2013.
\newblock URL \url{http://arxiv.org/abs/1306.2547}.

\bibitem[Gottlieb et~al.(2013b)Gottlieb, Kontorovich, and
  Krauthgamer]{KrauthgamerRegression}
L.-A. Gottlieb, A.~Kontorovich, and R.~Krauthgamer.
\newblock Efficient regression in metric spaces via approximate lipschitz
  extension.
\newblock In \emph{Similarity-Based Pattern Recognition}, volume 7953 of
  \emph{Lecture Notes in Computer Science}, pages 43--58. Springer Berlin
  Heidelberg, 2013b.
\newblock ISBN 978-3-642-39139-2.
\newblock \doi{10.1007/978-3-642-39140-8_3}.
\newblock URL \url{http://dx.doi.org/10.1007/978-3-642-39140-8_3}.

\bibitem[Jensen(1993)]{Jensen}
Robert Jensen.
\newblock Uniqueness of lipschitz extensions: Minimizing the sup norm of the
  gradient.
\newblock \emph{Archive for Rational Mechanics and Analysis}, 123\penalty0
  (1):\penalty0 51--74, 1993.
\newblock ISSN 0003-9527.
\newblock \doi{10.1007/BF00386368}.
\newblock URL \url{http://dx.doi.org/10.1007/BF00386368}.

\bibitem[Kirszbraun(1934)]{Kirszbraun}
M.~Kirszbraun.
\newblock Über die zusammenziehende und lipschitzsche transformationen.
\newblock \emph{Fundamenta Mathematicae}, 22\penalty0 (1):\penalty0 77--108,
  1934.
\newblock URL \url{http://eudml.org/doc/212681}.

\bibitem[Lazarus et~al.(1999)Lazarus, Loeb, Propp, Stromquist, and
  Ullman]{Lazarus}
Andrew~J. Lazarus, Daniel~E. Loeb, James~G. Propp, Walter~R. Stromquist, and
  Daniel~H. Ullman.
\newblock Combinatorial games under auction play.
\newblock \emph{Games and Economic Behavior}, 27\penalty0 (2):\penalty0 229 --
  264, 1999.
\newblock ISSN 0899-8256.
\newblock \doi{http://dx.doi.org/10.1006/game.1998.0676}.
\newblock URL
  \url{http://www.sciencedirect.com/science/article/pii/S0899825698906765}.

\bibitem[Ma and Fu(2011)]{Ma:2011:MLT:2207974}
Yunqian Ma and Yun Fu.
\newblock \emph{Manifold Learning Theory and Applications}.
\newblock CRC Press, Inc., Boca Raton, FL, USA, 1st edition, 2011.
\newblock ISBN 1439871094, 9781439871096.

\bibitem[McShane(1934)]{McShane}
E.~J. McShane.
\newblock Extension of range of functions.
\newblock \emph{Bull. Amer. Math. Soc.}, 40\penalty0 (12):\penalty0 837--842,
  12 1934.
\newblock URL \url{http://projecteuclid.org/euclid.bams/1183497871}.

\bibitem[Micchelli et~al.(1976)Micchelli, Rivlin, and Winograd]{WinogradOptLip}
C.A. Micchelli, T.J. Rivlin, and S.~Winograd.
\newblock The optimal recovery of smooth functions.
\newblock \emph{Numerische Mathematik}, 26\penalty0 (2):\penalty0 191--200,
  1976.
\newblock ISSN 0029-599X.
\newblock \doi{10.1007/BF01395972}.
\newblock URL \url{http://dx.doi.org/10.1007/BF01395972}.

\bibitem[Milman(1999)]{Milman}
V.~A. Milman.
\newblock Absolutely minimal extensions of functions on metric spaces.
\newblock 1999.
\newblock URL
  \url{http://iopscience.iop.org/1064-5616/190/6/A05/pdf/MSB_190_6_A05.pdf}.

\bibitem[Nadler et~al.(2009)Nadler, Srebro, and Zhou]{Nadler}
Boaz Nadler, Nathan Srebro, and Xueyuan Zhou.
\newblock Statistical analysis of semi-supervised learning: The limit of
  infinite unlabelled data.
\newblock 2009.
\newblock URL \url{http://ttic.uchicago.edu/~nati/Publications/NSZnips09.pdf}.

\bibitem[{Naor} and {Sheffield}(2010)]{Naor}
A.~{Naor} and S.~{Sheffield}.
\newblock {Absolutely minimal Lipschitz extension of tree-valued mappings}.
\newblock \emph{CoRR}, abs/1005.2535, May 2010.
\newblock URL \url{http://arxiv.org/abs/1005.2535}.

\bibitem[{Oberman}(2011)]{Oberman}
A.~M. {Oberman}.
\newblock {Finite difference methods for the Infinity Laplace and p-Laplace
  equations}.
\newblock \emph{CoRR}, abs/1107.5278, July 2011.
\newblock URL \url{http://arxiv.org/abs/1107.5278}.

\bibitem[Peres et~al.(2011)Peres, Schramm, Sheffield, and Wilson]{Peres}
Yuval Peres, Oded Schramm, Scott Sheffield, and DavidB. Wilson.
\newblock Tug-of-war and the infinity laplacian.
\newblock In \emph{Selected Works of Oded Schramm}, Selected Works in
  Probability and Statistics, pages 595--638. Springer New York, 2011.
\newblock ISBN 978-1-4419-9674-9.
\newblock \doi{10.1007/978-1-4419-9675-6_18}.
\newblock URL \url{http://dx.doi.org/10.1007/978-1-4419-9675-6_18}.

\bibitem[{Sheffield} and {Smart}(2010)]{Smart}
S.~{Sheffield} and C.~K. {Smart}.
\newblock {Vector-valued optimal Lipschitz extensions}.
\newblock \emph{CoRR}, abs/1006.1741, June 2010.
\newblock URL \url{http://arxiv.org/abs/1006.1741}.

\bibitem[Sinop and Grady(2007)]{SinopGrady}
Ali~Kemal Sinop and Leo Grady.
\newblock A seeded image segmentation framework unifying graph cuts and random
  walker which yields a new algorithm.
\newblock In \emph{Computer Vision, 2007. ICCV 2007. IEEE 11th International
  Conference on}, pages 1--8. IEEE, 2007.

\bibitem[Vazirani(2001)]{Vazirani01}
Vijay~V. Vazirani.
\newblock \emph{Approximation Algorithms}.
\newblock Springer-Verlag New York, Inc., New York, NY, USA, 2001.
\newblock ISBN 3-540-65367-8.

\bibitem[von Luxburg and Bousquet(2003)]{UlrikeLip}
Ulrike von Luxburg and Olivier Bousquet.
\newblock Distance-based classification with lipschitz functions.
\newblock In \emph{Learning Theory and Kernel Machines}, volume 2777 of
  \emph{Lecture Notes in Computer Science}, pages 314--328. Springer Berlin
  Heidelberg, 2003.
\newblock ISBN 978-3-540-40720-1.
\newblock \doi{10.1007/978-3-540-45167-9_24}.
\newblock URL \url{http://dx.doi.org/10.1007/978-3-540-45167-9_24}.

\bibitem[Whitney(1934)]{Whitney}
Hassler Whitney.
\newblock Analytic extensions of differentiable functions defined in closed
  sets.
\newblock \emph{Transactions of the American Mathematical Society}, 36\penalty0
  (1):\penalty0 pp. 63--89, 1934.
\newblock ISSN 00029947.
\newblock URL \url{http://www.jstor.org/stable/1989708}.

\bibitem[Zhou et~al.(2007)Zhou, Burges, and Tao]{Zhou2007}
Dengyong Zhou, Christopher J.~C. Burges, and Tao Tao.
\newblock Transductive link spam detection.
\newblock In \emph{Proceedings of the 3rd International Workshop on Adversarial
  Information Retrieval on the Web}, AIRWeb '07, pages 21--28, New York, NY,
  USA, 2007. ACM.
\newblock ISBN 978-1-59593-732-2.
\newblock \doi{10.1145/1244408.1244413}.
\newblock URL \url{http://doi.acm.org/10.1145/1244408.1244413}.

\bibitem[Zhu et~al.(2003)Zhu, Ghahramani, and Lafferty]{Zhu03}
Xiaojin Zhu, Zoubin Ghahramani, and John Lafferty.
\newblock Semi-supervised learning using gaussian fields and harmonic
  functions.
\newblock In \emph{IN ICML}, pages 912--919, 2003.

\end{thebibliography}
\appendix

\section{Basic Properties of Lex-Minimizers}
\label{app:lex_basics}
\subsection{Meta Algorithm}

\begin{mdframed}
\captionof{table}{\label{fig:meta-algorithm} Algorithm \meta: Given a well-posed instance $(G,v_{0})$, outputs  $\lex_G[v_{0}]$.\hfill \hfill}
    
     \FOR $i =1,2,\ldots:$
    \begin{tight_enumerate}
      \item \IF $T(v_{i-1}) = V_G,$ \THEN return $v_{i-1}$.
      \item $E' = E_G \setminus (T(v_{i-1}) \times T(v_{i-1}))$, $G' := (V_G,E')$.
      \item Let $P^\star_{i}$ be a steepest fixable path in
      $(G',v_{i-1}).$ Let
      $\alpha^\star_{i} \leftarrow \nabla P^\star(v_{i-1}).$
      \item $v_{i} \leftarrow \fixpath[v_{i-1},P^\star_{i}].$
      \end{tight_enumerate}
\end{mdframed}
In this subsection, we prove the results that appeared in section \ref{sec:lex_basics}. We start with a simple observation.

\begin{proposition}
\label{prop:basics:fix-well-posed}
Given a well-posed instance $(G,v_{0})$ such that $T(v_{0}) \neq V,$
let $P$ be a steepest fixable path in $(G,v_{0}).$ Then,
$\fixpath[v_{0},P]$ extends $v_{0},$ and $(G,\fixpath[v_{0},P])$ is
also a well-posed instance.
\end{proposition}
\Anote{include proof in the extended version}
The properties we prove
below do not depend on the choice of the steepest fixable path.
\begin{proposition}\label{prop:meta}
  For any well-posed instance $(G,v_{0}),$ with $|V_G| = n,$
  {\meta$(G,v_{0})$} terminates in at most $n$ iterations, and outputs
  a complete voltage assignment $v$ that extends $v_0.$
\end{proposition}
\begin{proofof}{of Proposition \ref{prop:meta}}
  By Proposition~\ref{prop:basics:fix-well-posed}, at any iteration
  $i,$ $v_{i-1}$ extends $v_0$ and $(G',v_{i-1})$ is a well-posed
  instance. {\meta} only outputs $v_{i-1}$ iff $T(v_{i-1}) = V,$ which
  means $v_{i-1}$ is a complete voltage assignment. For any $v_{i-1}$
  that is not complete, for any $x \in V \setminus T(v_{i-1}),$ we must
  have a free terminal path in $(G',v_{i-1})$ that contains $x$. 
    Hence, a steepest fixable path $P^\star_{i}$
  exists in $(G',v_{i-1}).$ Since $P^{\star}_{i}$ is a free terminal
  path, $\fixpath[v_{i-1},P^\star_{i}]$ fixes the voltage for at least
  one non-terminal. Thus, {\meta}$(G,v_{0})$ must complete in at most $n$
  iterations.
\end{proofof}

For the following lemmas, consider a run of {\meta} with well-posed instance
$(G,v_{0})$ as input. Let $v_{\mathsf{out}}$ be the complete voltage assignment
output by {\meta}. 
Let $E_i$ be the set of edges  $E'$ and $G_i$ be the graph  $G'$  constructed in iteration $i$ of \meta. 
\begin{lemma}
\label{lem:basics:no-steeper-path}
For every edge $e \in E_{i-1} \setminus E_{i},$ we have
$\abs{\grad[v_{\mathsf{out}}](e)} \le \alpha_{i}^{\star}.$ Moreover,
$\alpha^{\star}_{i}$ is non-increasing with $i.$
\end{lemma}
%
\begin{proofof}{of Lemma \ref{lem:basics:no-steeper-path}}
  Let $P^{\star}_{i} = (x_{0},\ldots,x_{r})$ be a steepest fixable
  path in iteration $i$ (when we deal with instance $(G_{i-1},v_{i-1})$). Consider a terminal path $P_{i+1}$ in
  $(G_{i},v_{i})$ such that
  $\{ \partial_{0} P_{i+1}, \partial_{1} P_{i+1}\} \cap (T(v_{i})
  \setminus T(v_{i-1})) \neq \emptyset.$
  We claim that $\nabla P_{i+1}(v_{i}) \le \alpha^{\star}_{i}.$ On the
  contrary, assume that $\nabla P_{i+1}(v_{i}) > \alpha^{\star}_{i}.$
  Consider the case
  $\partial_{0} P_{i+1} \in T(v_{i}) \setminus
  T(v_{i-1}), \partial_{1} P_{1} \in T(v_{i-1}).$
  By the definition of $v_{i},$ we must have
  $\partial_{0} P_{i+1} = x_{j}$ for some $j \in [r-1].$ Let
  $P^{\prime}_{i+1}$ be the path formed by joining paths
  $P^{\star}_{i}[x_0:x_j]$ and $P_{i+1}.$ $P^{\prime}_{i+1}$ is a
  free terminal path in $(G_{i-1},v_{i-1}).$ We have,
\begin{align*}
v_{i-1}(x_{0}) - v_{i-1}(\partial_{1} P_{i+1}) & = (v_{i}(x_{0}) -
v_{i}(x_{j})) + (v_{i}(\partial_{0} P_{i+1})- v_{i}(\partial_{1} P_{i+1})) \\
& > \alpha^{\star}_{i} \cdot \len(P_{i}^{\star}[x_{0}:x_{j}]) + \alpha^{\star}_{i} \cdot \len(P_{i+1}) =
  \alpha^{\star}_{i} \cdot \len(P_{i+1}^{\prime}),
\end{align*}
giving $\nabla P^{\prime}_{i+1} (v_{i}) > \alpha^{\star}_{i},$ which is a
contradiction since the steepest fixable path $P^{\star}_i$ in
$(G_{i-1},v_{i-1})$ has gradient $\alpha^{\star}_{i}$. The other cases can be
handled similarly.

Applying the above claim to an edge $e \in E_{i-1} \setminus E_{i},$
whose gradient is fixed for the first time in iteration $i,$ we obtain
that $\grad[v_{i+1}](e) \le \alpha_{i}^{\star}.$ If $v$ is the
complete voltage assignment output by {\meta}, since $v$ extends
$v_{i+1},$ we get $\grad[v_{\mathsf{out}}](e) \le \alpha^{\star}_{i}$.
Applying the claim to the symmetric edge, we obtain
$-\grad[v_{\mathsf{out}}](e) \le \alpha_{i}^{\star},$ implying
$|\grad[v_{\mathsf{out}}](e)| \le \alpha_{i}^{\star}.$

Consider any free terminal path $P_{i+1}$ in $(G_{i},v_{i}).$
If $P_{i+1}$ is also a terminal path in
$(G_{i-1},v_{i-1}),$ it is a free terminal path in $(G_{i-1},v_{i-1}).$ In addition, since a steepest fixable path $P^{\star}_{i}$ in
$(G_{i-1},v_{i-1})$ has $\nabla P^{\star}_{i} = \alpha^{\star}_{i},$ we get
$\nabla P_{i+1}(v_{i}) = \nabla P_{i+1}(v_{i-1}) \le
\alpha_{i}^\star.$
Otherwise, we must have
$\{ \partial_{0} P_{i+1}, \partial_{1} P_{i+1}\} \cap (T(v_{i})
\setminus T(v_{i-1})) \neq \emptyset,$  
and we can deduce $\nabla P_{i+1}(v_{i}) \le \alpha_{i}^\star$ using
the above claim. Thus, all free terminal paths $P_{i+1}$
in $(G_{i},v_{i})$ satisfy $\nabla P_{i+1}(v_{i}) \le \alpha^{\star}_{i}.$
In particular,
$\alpha^{\star}_{i+1} = \nabla P^\star_{i+1}(v_{i}) \le
\alpha^{\star}_{i}.$
Thus, $\alpha^{\star}_{i}$ is non-increasing with $i.$
\end{proofof}

\begin{lemma}
  \label{lem:basics:lex-minimal}
  For any complete voltage assignment $v$ for $G$ that extends
  $v_{0},$ if $v \neq v_{\mathsf{out}},$ we have
  $\grad[v] \nlexless \grad[v_{\mathsf{out}}],$ and hence
  $\grad[v_{\mathsf{out}}] \lexless \grad[v].$
\end{lemma}
\begin{proofof}{of Lemma \ref{lem:basics:lex-minimal}}
  Consider any complete voltage assignment $v$ for $G$ that extends
  $v_{0},$ such that $v \neq v_{\mathsf{out}}.$ Thus, there exists a
  unique $i$ such that $v$ extends $v_{i-1}$ but does not extend
  $v_{i}.$ We will argue that
  $\grad[v] \nlexless \grad[v_{\mathsf{out}}],$ and hence
  $\grad[v_{\mathsf{out}}] \lexless \grad[v].$ For every edge
  $e \in E \setminus E_{i-1}$ that has been fixed so far,
  $\grad[v](e) = \grad[v_{i-1}](e) = \grad[v_{\mathsf{out}}](e),$ and
  hence we can ignore these edges.

  Since $v$ extends $v_{i-1}$ but not $v_{i},$ there exists an
  $x \in T(v_{i}) \setminus T(v_{i-1})$ such that
  $v(x) \neq v_{i}(x) = v_{\mathsf{out}}(x).$ Assume $v(x) < v_{i}(x)$
  (the other case is symmetric). If $P^{\star}_{i} = (x_0,\ldots,x_r)$
  is the steepest fixable path with gradient $\alpha^{\star}_{i}$
  picked in iteration $i,$ we must have $x = x_j$ for some
  $j \in [r-1].$ Thus,
\begin{align*}
  \sum_{k=1}^{j} (v(x_{k-1}) - v(x_{k})) = v(x_0) - v(x_j) 
 >v_{i}(x_0) - v_{i}(x_j)  = \alpha^{\star}_{i} \cdot \len(P^{\star}_{i}[x_0:x_j]) =
\alpha^{\star}_{i} \cdot \sum_{k=1}^j \len(x_{k-1},x_{k}).
\end{align*}
Thus, for some $k \in [j],$ we must have
$\grad[v](x_{k-1},x_{k}) > \alpha_{i}^{\star}.$ Since $P^{*}_i$ is a path in $G_{i-1}$, we have $\{ x_{k-1},x_{k} \} \not \subseteq T(v_{i-1})$. This gives 
$(x_{k-1},x_{k}) \in (E_{i-1} \setminus E_{i}).$ 
 But then,
from Lemma~\ref{lem:basics:no-steeper-path}, it follows that for all
$e \in  (E_{i-1} \setminus E_{i}),$ we have
$|\grad[v_{\mathsf{out}}](e)| \le \alpha_{i}^{\star}.$ Thus, we have
$\grad[v] \nlexless \grad[v_{\mathsf{out}}].$
\end{proofof}

\begin{lemma}
\label{lem:gradSteepestPath}
Let 
$P = (x_{0},\ldots,x_r)$ be a steepest fixable path such that it does not have any edges in $T(v_0) \times T(v_0)$ and 
$v_1 \defeq \fixpath_{G}[v_{0},P]$. Then for every $i \in [r],$ we
have $\grad[v_1](x_{i-1},x_{i}) = \nabla P.$
\end{lemma}

\begin{proofof}{ of Lemma \ref{lem:gradSteepestPath}}
Suppose this is not true and let $j \in [r]$ be the minimum number such that $\grad[v_1](x_{j-1},x_{j}) \neq \nabla P.$ By definition of $v_1$ we would necessarily have $j<r$ and $v_{j} \in  T(v_0).$ Suppose  $\grad[v_1](x_{j-1},x_{j}) < \nabla P.$ We would then have $v_1(x_0) - v_1(x_j) < \nabla P * \len (P[x_0:x_j]).$  Since $P$ does not have any edges in $T(v_0) \times T(v_0)$,  $P_1 := (x_j,...,x_r)$ would be a free terminal path with $\nabla P_1 > \nabla P.$ This is a contradiction. Other cases can be ruled out similarly. 

\end{proofof}

\begin{proofof}{ of Theorem \ref{thm:basics:meta}}
  Consider an arbitrary run of {\meta} on $(G,v_0).$ Let
  $v_{\mathsf{out}}$ be the complete voltage assignment output by
  {\meta}. Proposition~\ref{prop:basics:fix-well-posed} implies that
  $v_{\mathsf{out}}$ extends $v_{0}.$
  Lemma~\ref{lem:basics:lex-minimal} implies that for any complete
  voltage assignment $v \neq v_{\mathsf{out}}$ that extends $v_0,$ we
  have   $\grad[v_{\mathsf{out}}] \lexless \grad[v].$ Thus,
  $v_{\mathsf{out}}$ is a lex-minimizer. Moreover, the lemma also
  gives that for any such $v,$ 
$\grad[v] \nlexless \grad[v_{\mathsf{out}}].$ 
 and hence $v_{\mathsf{out}}$ is a unique lex-minimizer.
  Thus,
  $v_{\mathsf{out}}$ is the unique voltage assignment satisfying
  Def.~\ref{def:lex-min}, and we denote it as $\lex_G[v_{0}].$ Since
  we started with an arbitrary run of {\meta}, uniqueness implies that
  every run of {\meta} on $(G,v_{0})$ must output $\lex_G[v_{0}].$
\end{proofof}

\begin{proofof}{of Lemma \ref{lem:basics:inf-duality}}
  Suppose we have a complete voltage assignment $v$ extending $v_{0},$
  such that $\norm{\grad[v]}_{\infty} \le \alpha.$ For any terminal
  path $P = (x_0,\ldots,x_r),$ we get,
  \[\nabla P(v_{0}) = v_{0}(\partial_{0} P) - v_{0}(\partial_{1} P) =
  v(\partial_{0} P) - v(\partial_{1} P) = \sum_{i=1}^{r}
  \grad[v](x_{i-1},x_{i}) \le \alpha \cdot \sum_{i=1}^{r}
  \len(x_{i-1},x_{i}) = \alpha \cdot \ell(P),\]
giving $\nabla P(v_{0}) \le \alpha.$ 

On the other hand, suppose every terminal path $P$ in $(G,v_{0})$
satisfies $\nabla P(v_{0}) \le \alpha.$ Consider
$v \defeq \lex_G[v_{0}].$ We know that $v$ extends $v_{0}.$ For
every edge $e \in E_{G} \cap T(v_{0}) \times T(v_{0}),$ $e$ is a (trivial)
terminal path in $(G,v_{0}),$ and hence has satisfies
$\grad[v](e) = \grad[v_{0}](e) = \nabla {e}(v_{0}) \le \alpha.$
Considering the reverse edge, we also obtain
$-\grad[v](e) \le \alpha.$ Thus, $|\grad[v](e)| \le \alpha.$ Moreover,
using Lemma~\ref{lem:basics:no-steeper-path}, we know that for edge
$e \in E_G \setminus T(v_{0}) \times T(v_{0}),$
$|\grad[v](e)| \le \alpha_{1}^{\star} = \nabla P_{1}^{\star} \le
\alpha$
since $P_{1}$ is a terminal path in $(G,v_{0}).$ Thus, for every
$e \in E_G,$ $|\grad[v](e)| \le \alpha,$ and hence
$\norm{\grad[v]}_{\infty} \le \alpha.$
\end{proofof}

\subsection{Stability}
In this subsection, we sketch a proof of the monotonicity of lex-minimizers and show how it implies the stability property claimed earlier.

For any well-posed $(G,v_{0}),$ there could be several possible
executions of {\meta}, each characterized by the sequence of paths
$P^{\star}_i.$ We can apply Theorem~\ref{thm:basics:meta} to deduce
the following structural result about the lex-minimizer.
\begin{corollary}
  For any well-posed instance $(G,v_{0}),$ consider a sequence of paths
  $(P_1,\ldots,P_r)$ and voltage assignments $(v_1, \ldots, v_r)$ for
  some positive integer $r$ such that:
\begin{enumerate}
\item $P^{\star}_i$ is a steepest fixable path in $(G_{i-1},v_{i-1})$ for
  $i=1,\ldots,r.$
\item $v_i = \fixpath[v_{i-1},P^{\star}_i]$ for $i=1,\ldots,r.$
\item $T(v_r) = V_{G}.$
\end{enumerate}
Then, we have $v_r = \lex_G[v_{0}].$
\end{corollary}

We call such a sequence of paths and voltages to be a decomposition of
$\lex_G[v_{0}].$ Again, note that $\lex_G[v_{0}]$ can possibly
have multiple decompositions. However, any two such decompositions are
\emph{consistent} in the sense that they produce the same voltage
assignment.

\newcommand\vecone{\boldsymbol{\mathit{1}}}

\begin{proofof}{of Corollary~\ref{cor:stability}}
 We first define some operations on partial assignments which simplifies the notation. Let   $v_{0},v_{1}$ be any two partial assignments with the same set of terminals $T := T(v_{0}) = T(v_{1})$ and $c,d \in \rea$.  By $cv_{0} +d$ we mean a partial assignment $v$ with $T(v) = T$ satisfying  $v(t) = cv_{0}(t) +d$ for all $t \in T$. Also, by $v_{0} +v_{1}$ we mean a partial assignment $v$ with $T(v) = T$ satisfying $v(t) = v_{0}(t) +v_{1}(t)$ for all $t \in T.$ Also, we say $v_{1} \ge v_{0}$ if $v_{1}(t) \ge v_{0}(t)$ for all $t \in T$.
  
%

Now we can show how Corollary \ref{cor:stability} follows from Theorem \ref{thm:monotonicity}. Let $v := v_{1} - v_{0}$, and $\nbr{v}_{\infty} = \epsilon$, for some $\epsilon >0.$ 
Therefore, 
$v_{0} + \epsilon  \ge v_{1} \ge v_{0} - \epsilon.$  
 Theorem \ref{thm:monotonicity} then implies that 
 $\lex_G[v_{0}] + \epsilon \ge \lex[v_{1}] \ge \lex[v_{0}] - \epsilon,$   
hence proving the corollary.
\end{proofof}

\begin{proofof}{sketch of Theorem \ref{thm:monotonicity}}
 It is easy to see that the first statement holds. For the second statement, we first observe that if there is a sequence of paths $P_1,...,P_r$ that is simultaneously a  decomposition of both $\lex[v_{0}]$ and $\lex[v_{1}]$, then this is easy to see. If such a path sequence doesn't exist, then we look at $v_t := v_{0} + t(v_{1} - v_{0})$.  We state here without a proof (though the proof is elementary) that we can then split the interval $[0,1]$ into finitely many subintervals $[a_0,a_1],[a_1,a_2],..,[a_{k-1},a_k]$, with  $a_0 = 0, a_k = 1$,  such that for any $i$, there is a path sequence $P_1,...,P_r$  which is a decomposition of $\lex[v_t]$ for all $t \in [a_i,a_{i+1}]$. We then observe that 
 $v_{0} = v_{a_0} \le v_{a_1} \le ... v_{a_k} = v_{1}.$ 
 Since for every $a_i, a_{i+1}$, there is a path sequence which is simultaneously a decomposition of both  $\lex[v_{a_i}]$ and $\lex[v_{a_{i+1}}]$, we immediately get 
 $$\lex[v_{0}] = \lex[v_{a_0}] \le \lex[v_{a_1}] \le ... \le \lex[v_{a_k}] = \lex[v_{1}].$$ 
 \end{proofof}
 \Anote{full proof in the extended version}

\subsection{Alternate Characterizations}
\begin{proofof}{of Theorem~\ref{thm:basics:max-min}}
  We know that $\lex_G[v_{0}]$ extends $v_{0}.$ We first prove that
  $v \defeq \lex_G[v_{0}]$ satisfies the max-min gradient averaging
  property. Assume to the contrary. Thus, there exists
  $x \in V_G \setminus T(v_{0})$ such that
 \[\max_{y:(x,y) \in E_{G}} \grad[v](x,y) \neq - \min_{y:(x,y) \in E_{G}}
 \grad[v](x,y).\]
 Assume that
 $\max_{(x,y) \in E_{G}} \grad[v](x,y) \ge - \min_{(x,y) \in E_{G}}
 \grad[v](x,y).$
 Then, consider $v^\prime$ extending $v_{0}$ that is identical to $v$
 except for $v^{\prime}(x) = v(x) - \eps$ for $\eps > 0.$ For $\eps$
 small enough, we get that 
 \[\max_{y:(x,y) \in E_{G}} \grad[v^\prime](x,y) < \max_{y:(x,y)
  \in E_{G}} \grad[v](x,y)\]
  and
 \[- \min_{y:(x,y) \in E_{G}} \grad[v^\prime](x,y) < \max_{y:(x,y)
  \in E_{G}} \grad[v](x,y).\] 
  
The gradient of edges not incident on the vertex $x$ is left unchanged. This implies that $\grad[v] \nlexless \grad[v^{\prime}],$ contradicting
the assumption that $v$ is the lex-minimizer. (The other case is
similar).

For the other direction. Consider a complete voltage assignment $v$
extending $v_{0}$ that satisfies the max-min gradient averaging
property w.r.t. $(G,v_{0}).$ Let
$$\alpha \defeq \max_{\substack{(x,y) \in E_G \\ x \in V \setminus T(v_{0})}} \grad[v](x,y) \ge 0 $$ 
be the maximum edge gradient, and consider any edge
$(x_{0},x_{1}) \in E_{G}$ such that
$\grad[v](x_{1},x_{0}) = \alpha,$ with
$x_{1} \in V \setminus T(v_{0}).$ If $\alpha = 0,$ $\grad[v]$ is
identically zero, and is trivially the lex-minimal gradient
assignment. Thus, both $v$ and $\lex_G[v_{0}]$ are constant on each
connected component. Since $(G,v_{0})$ is well-posed, there is at
least one terminal in each component, and hence $v$ and
$\lex_G[v_{0}]$ must be identical.

Now assume $\alpha > 0.$ By the max-min gradient averaging property,
$\exists x_{2} \in V_G$ such that $(x_{1},x_{2}) \in E_G$ and
\begin{align*}
  \grad[v](x_{1},x_{2}) 
  = \min_{y:(x_{1},y) \in E_G} \grad[v](x_{1},y) 
  & =  -\max_{y:(x_{1},y) \in E_G} \grad[v](x_{1},y) \\ 
& \le -  \grad[v](x_{1},x_{0}) 
  = -\alpha.
\end{align*}
Thus, $\grad[v](x_{2},x_{1}) \ge \alpha.$ Since $\alpha$ is the
maximum edge gradient, we must have
$\grad[v](x_{2},x_{1}) = \alpha.$ Moreover,
$v(x_{2}) > v(x_{1}) > v(x_{0}),$ thus $x_{2} \neq x_{0}.$ We can
inductively apply this argument at $x_{2}$ until we hit a
terminal. Similarly, if $x_{0} \notin T(v_{0})$ we can extend the path
in the other direction. Consequently, we obtain a path
$P = (x_{j},\ldots,x_{2},x_{1},x_{0},x_{-1},\ldots,x_{k})$ with all
vertices as distinct, such that $x_{j},x_{k} \in T(v_{0}),$ and
$x_{i} \in V \setminus T(v_{0})$ for all $i \in [j+1,k-1].$ Moreover,
$\grad[v](x_{i},x_{i-1}) = \alpha$ for all $j < i \le k.$ Thus, $P$ is
a free terminal path with $\nabla P[v_0] = \alpha.$

Moreover, since $v$ is a voltage assignment extending $v_{0}$ with
$\norm{\grad[v]}_{\infty} = \alpha,$ using
Lemma~\ref{lem:basics:inf-duality}, we know that every terminal path
$P^{\prime}$ in $(G,v_{0})$ must satisfy
$\nabla P^{\prime}(v_{0}) \le \alpha.$ Thus, $P$ is a steepest fixable
path in $(G,v_{0}).$ Thus, letting $v_1 \defeq \fixpath[v_{0},P],$
using Corollary~\ref{cor:fix-path}, we obtain that
$\lex_{G}[v_{1}] = \lex_{G}[v_{0}].$ Moreover, since
$\alpha = \nabla P[v_{0}] = \grad[v](x_{i},x_{i-1})$ for all
$i \in (j,k],$ we get $v_1(x_i) = v(x_i)$ for all $i \in (j,k).$ Thus,
$v$ extends $v_1.$

We can iterate this argument for $r$ iterations until
$T(v_{r}) = V_G,$ giving $v = v_{r}$ and
$v_{r} = \lex_{G}[v_{r}] = \lex_G[v_{0}].$ (Since we are fixing at
least one terminal at each iteration, this procedure
terminates). Thus, we get $v = \lex_G[v_{0}].$
\end{proofof}


\section{Description of the Algorithms}


\begin{mdframed}
\vspace{\algtopspace}
\captionof{table}{ \label{alg:dijkstra}\moddijkstra$(G,v_{0},\alpha)$:
  Given a well-posed instance $(G,v_{0}),$ a gradient value
  $\alpha \ge 0,$ outputs a complete voltage assignment $v$ for $G,$
  and an array $\parent : V \to V \cup \{\mathsf{null}\}.$ }
    \vspace{\algpostcaptionspace}
    \begin{tight_enumerate}
      \item \FOR $x \in V_G,$ 
      \item \hspace{\pgmtab} Add $x$ to a fibonacci heap, with ${\mathsf{key}}(x) = +\infty.$ 
      \item \hspace{\pgmtab} $\mathsf{finished}(x) \leftarrow \FALSE$
      \item \FOR $x \in T(v_{0})$
      \item \hspace{\pgmtab} Decrease ${\mathsf {key}}(x)$ to $v_{0}(x).$ 
      \item \hspace{\pgmtab} $\parent(x) \leftarrow \mathsf{null}.$
      \item \WHILE heap is not empty
      \item \hspace{\pgmtab} $x \leftarrow$ pop element with minimum $\mathsf{key}$ from heap
      \item \hspace{\pgmtab} $v(x) \leftarrow {\mathsf{key}}(x).$ ${\mathsf{finished}}(x) \leftarrow \TRUE.$
      \item \hspace{\pgmtab} \FOR $y  : (x,y) \in E_{G}$
      \item \hspace{\pgmtab}\hspace{\pgmtab} \IF ${\mathsf{finished}}(y) = \FALSE$
      \item \hspace{\pgmtab} \hspace{\pgmtab}\hspace{\pgmtab} \IF ${\mathsf{key}}(y) > v(x) + \alpha\cdot\len(x,y)$
      \item \hspace{\pgmtab} \hspace{\pgmtab}\hspace{\pgmtab}\hspace{\pgmtab} 
      Decrease ${\mathsf {key}}(y)$ to $v(x) + \alpha\cdot\len(x,y).$ 
      \item \hspace{\pgmtab} \hspace{\pgmtab}\hspace{\pgmtab}\hspace{\pgmtab} 
      $\parent(y) \leftarrow x.$ 
      \item \RETURN $(v,\parent)$
      \end{tight_enumerate}
\end{mdframed}
\begin{theorem}
  For a well-posed instance $(G,V_{0})$ and a gradient value
  $\alpha \ge 0,$ let
  $(v,\parent) \leftarrow$ \moddijkstra$(G,v_{0},\alpha).$ Then, $v$
  is a complete voltage assignment such that, 
  $\forall x \in V_{G},$
  $v(x) = \min_{t \in T(v_{0})} \{ v_{0}(t) + \alpha\dis(x,t)\}.$
Moreover, the pointer array $\parent$ satisfies $\forall x \notin T(v_{0}),$
$\parent(x) \neq {\mathsf{null}}$ and $v(x) = v(\parent(x)) +
\alpha\cdot \len(x,\parent(x)).$ 
\end{theorem}

\begin{mdframed}
  \vspace{\algtopspace}
  \captionof{table}{\label{alg:vlow} Algorithm
    \compvlow$(G,v_{0},\alpha)$: Given a well-posed instance
    $(G,v_{0}),$ a gradient value $\alpha \ge 0,$ outputs
    $\vLow,$ a complete voltage assignment for $G,$ and an
    array $\LParent$ $: V \to V \cup \{\mathsf{null}\}$.}
  \vspace{\algpostcaptionspace}
    \begin{tight_enumerate}
      \item $(\vLow,\LParent) \leftarrow$ \moddijkstra$(G, v_0,\alpha)$
      \item \RETURN  $(\vLow,\LParent)$
      \end{tight_enumerate}
   \end{mdframed}

\begin{mdframed}
  \vspace{\algtopspace}
  \captionof{table}{\label{alg:vhigh} Algorithm
    \compvhigh$(G,v_{0},\alpha)$: Given a well-posed instance
    $(G,v_{0}),$ a gradient value $\alpha \ge 0$, outputs
    $\vHigh,$ a complete voltage assignment for $G$, and
    an array $\HParent : V \to V \cup \{\mathsf{null}\}$.}
    \vspace{\algpostcaptionspace}
    \begin{tight_enumerate}
      \item \FOR $x \in  V_{G}$
      \item \hspace{\pgmtab} \IF $x \in T(v_{0})$ \THEN $v_{1}(x) \leftarrow -v_{0}(x)$ \ELSE $v_{1}(x) \leftarrow v_{1}(x).$
      \item $(\temp,\HParent) \leftarrow$ \moddijkstra$(G, v_1,\alpha)$
      \item \FOR $x \in V_{G} : \vHigh (x) \leftarrow - \temp(x)$
      \item \RETURN $(\vHigh,\HParent)$
      \end{tight_enumerate}
  \end{mdframed}
  
\begin{corollary}
  For a well-posed instance $(G,V_{0})$ and a gradient value
  $\alpha \ge 0,$ let $(\vLow[\alpha],\LParent) \leftarrow$
  \compvlow$(G,v_{0},\alpha)$ and
  $(\vHigh[\alpha],\HParent) \leftarrow$ \compvhigh$(G,v_{0},\alpha).$
  Then, $\vLow[\alpha], \vHigh[\alpha]$ are complete voltage
  assignments for $G$ such that, $\forall x \in V_{G},$
\begin{align*}
  \vLow[\alpha](x) = \min_{t \in T(v_{0})} \{ v_{0}(t) + \alpha\cdot \dis(x,t)\} &&
  \vHigh[\alpha](x) = \max_{t \in T(v_{0})} \{ v_{0}(t) - \alpha\cdot \dis(t,x)\}.
\end{align*}
 Moreover, the pointer arrays $\LParent, \HParent$ satisfy $\forall x \notin T(v_{0}),$
$\LParent(x),\HParent(x) \neq {\mathsf{null}}$ and 
\begin{align*}
\vLow[\alpha](x) & = \vLow[\alpha](\LParent(x)) +
\alpha\cdot \len(x,\LParent(x)), \\
\vHigh[\alpha](x) & = \vHigh[\alpha](\HParent(x)) -
\alpha\cdot \len(x,\HParent(x)).
\end{align*}
\end{corollary}


\begin{mdframed}
\vspace{\algtopspace}
\captionof{table}{ \label{alg:comp-inf-min} Algorithm \compinfmin $(G,v_{0})$: Given a well-posed instance $(G,v_{0})$, outputs a complete voltage assignment $\mathsf{{v}}$ for $G,$ extending $v_{0}$ that minimizes $\norm{\grad[\mathsf{{v}}]}_{\infty}$.}
\vspace{\algpostcaptionspace}
    \begin{tight_enumerate}
      \item $\alpha \leftarrow \max\{|\grad[v_{0}](e)| \ |\  e \in E_{G}
        \cap (T(v_{0}) \times T(v_{0})) \}.$
      \item $E_{G} \leftarrow E_{G} \setminus (T(v_{0}) \times T(v_{0}))$
      \item $P \leftarrow $\steepestpath$(G,v_{0}).$
      \item $\alpha \leftarrow \max\{\alpha,\nabla P(v_{0})\}$
      \item  $(\vLow,\LParent) \leftarrow$ \compvlow$(G,v_{0},\alpha)$
      \item $(\vHigh,\HParent) \leftarrow$ \compvhigh$(G,v_{0},\alpha)$
      \item \FOR $x \in V_{G}$ 
      \item \hspace{\pgmtab} \IF $x \in T(v_{0})$ 
      \item \hspace{\pgmtab} \hspace{\pgmtab} \THEN $\mathsf{v}(x)
        \leftarrow v_{0}(x)$ 
        \item \hspace{\pgmtab} \hspace{\pgmtab} \ELSE $\mathsf{{v}}(x) \leftarrow \frac{1}{2} \cdot (\vLow(x) + \vHigh(x)).$
      \item \RETURN $\mathsf{{v}}$
      \end{tight_enumerate}
  \end{mdframed}

  
  \begin{mdframed}
\vspace{\algtopspace}
\captionof{table}{\label{alg:comp-high-press-graph} Algorithm
  \comphighpressgraph $(G,v_{0},\alpha)$:
  Given a well-posed instance $(G,v_{0}),$
  a gradient value $\alpha \ge 0$, outputs a minimal induced subgraph $G^{\prime}$ of
  $G$ where every vertex has $\pressure[v_{0}](\cdot) > \alpha.$ }
\vspace{\algpostcaptionspace}
    \begin{tight_enumerate}
      \item  $(\vLow,\LParent) \leftarrow$ \compvlow$(G,v_{0},\alpha)$
      \item  $(\vHigh,\HParent) \leftarrow$ \compvhigh$(G,v_{0},\alpha)$
      \item $V_{G^{\prime}} \leftarrow \{ x \in V_G\ |\ \vHigh(x) > \vLow(x)\ \}$
      \item $E_{G^{\prime}} \leftarrow \{ (x,y) \in E_G\ |\ x,y \in
        V_{G^{\prime}} \}.$
      \item $G^{\prime} \leftarrow (V^{\prime}, E^{\prime}, \len)$
      \item \RETURN $G^{\prime}$
      \end{tight_enumerate}
  \end{mdframed}
  
  \begin{proofof}{of Lemma \ref{lem:vlowvhigh-pressure}}
\[
\vHigh[\alpha](x) >\vLow[\alpha](x)
\]
is equivalent to
 \[
 \max_{t \in T(v_{0})} \{ v_{0}(t) - \alpha\cdot \dis(t,x)\} > \min_{t \in T(v_{0})} \{ v_{0}(t) + \alpha\cdot \dis(x,t)\},
\]
which implies that there exists terminals $s,t \in T(v_0)$ such that
\[
 v_{0}(t) - \alpha\cdot \dis(t,x) > v_{0}(s) + \alpha\cdot \dis(x,s)
\]
thus, 
\[
\pressure[v_{0}](x) \geq \frac{ v_{0}(t) - v_{0}(s) }{\dis(t,x)+ \dis(x,s)} > \alpha.
\]
So the inequality on \vHigh\, and \vLow\, implies that pressure is strictly greater than $\alpha$.
On the other hand, if $\pressure[v_{0}](x) > \alpha$, there exists terminals $s,t \in T(v_0)$ such that
\[
\frac{ v_{0}(t) - v_{0}(s) }{\dis(t,x)+ \dis(x,s)} = \pressure[v_{0}](x) > \alpha.
\]
Hence, 
\[
 v_{0}(t) - \alpha\cdot \dis(t,x) > v_{0}(s) + \alpha\cdot \dis(x,s)
\]
which implies $\vHigh[\alpha](x) >\vLow[\alpha](x)$.
\end{proofof}
  
  \begin{mdframed}
\vspace{\algtopspace}
\captionof{table}{  \label{alg:steepest-path} Algorithm \steepestpath$(G,v_{0})$: Given  a well-posed instance $(G,v_{0}),$ with 
$T(v_{0}) \neq V_{G},$ 
outputs a steepest free terminal path $P$ in $(G,v_{0}).$}
    \vspace{\algpostcaptionspace}
    \begin{tight_enumerate}
      \item Sample uniformly random $e \in E_{G}.$ Let $e = (x_{1},x_{2}).$
      \item Sample uniformly random $x_{3} \in V_{G}.$
      \item \FOR $i=1$ \TO $3$
      \item \hspace{\pgmtab} $P
        \leftarrow $ \vertexsteepestpath$(G,v_{0},x_{i})$
        \item Let $j \in \argmax_{j \in \{1,2,3\}} \nabla P_{j}(v_{0})$
      \item $G^{\prime} \leftarrow $ \comphighpressgraph$(G,v_{0},\nabla P_{j}(v_{0}))$
      \item \IF $E_{G^{\prime}} = \emptyset,$ 
      \item \hspace{\pgmtab} \THEN \RETURN $P_{j}$ 
      \item \hspace{\pgmtab} \ELSE \RETURN
        \steepestpath$(G^{\prime},v_{0}|_{V_{G^{\prime}}})$
      \end{tight_enumerate}
  \end{mdframed}

 \begin{mdframed}
\vspace{\algtopspace}
\captionof{table}{ \label{alg:comp-lex-min} Algorithm \complexmin$(G,v_{0})$: Given a well-posed instance $(G,v_{0}),$ with $T(v_{0}) \neq V_{G}$, outputs  $\lex_G[v_{0}]$.}
    \vspace{\algpostcaptionspace}
    \begin{tight_enumerate}
      \item \WHILE $T(v_{0}) \neq V_{G}$
      \item \hspace{\pgmtab} $E_{G} \leftarrow E_{G} \setminus
        (T(v_{0}) \times T(v_{0}))$
      \item \hspace{\pgmtab} $P \leftarrow $ \steepestpath$(G,v_{0})$
      \item \hspace{\pgmtab} $v_{0} \leftarrow $ \fixpath$[v_{0},P]$
      \item \RETURN $v_{0}$
      \end{tight_enumerate}
 \end{mdframed}

  \begin{mdframed}
\vspace{\algtopspace}
\captionof{table}{\label{alg:vertex-steepest-path} Algorithm
  \vertexsteepestpath$(G,v_{0},x)$: Given a well-posed instance
  $(G,v_{0}),$ and a vertex $x \in V_{G}$, outputs a steepest terminal
  path in $(G,v_{0})$ through $x$.}
    \vspace{\algpostcaptionspace}
    \begin{tight_enumerate}
      \item Using Dijkstra's algorithm, compute $\dis(x,t)$ for all
        $t \in T(v_{0})$
        \item \IF $x \in T(v_{0})$
        \item \hspace{\pgmtab} $y \leftarrow \argmax_{y \in T(v_{0})}
          \frac{|v_{0}(x)-v_{0}(y)|}{\dis(x,y)}$
        \item \hspace{\pgmtab} \IF $v_0(x) \ge v_{0}(y)$
        \item \hspace{\pgmtab} \hspace{\pgmtab} \THEN \RETURN a
          shortest path from $x$ to $y$ 
        \item \hspace{\pgmtab} \hspace{\pgmtab} \ELSE \RETURN a
          shortest path from $y$ to $x$ 
        \item \ELSE
          \item \hspace{\pgmtab} \FOR $t \notin T(v_{0}),$
          $\mathsf{d}(t) \leftarrow \dis(x,t)$
      \item \hspace{\pgmtab} $(t_{1},t_{2}) \leftarrow$
        \starsteepestpath$(T(v_{0}),v_{0}|_{T(v_{0})},\mathsf{d})$
      \item \hspace{\pgmtab} Let $P_{1}$ be a shortest path from $t_{1}$ to $x.$ Let
        $P_{2}$ be a shortest path from $x$ to $t_{2}.$
      \item \hspace{\pgmtab} $P \leftarrow (P_{1},P_{2}).$ \RETURN $P.$
      \end{tight_enumerate}
 \end{mdframed}

\begin{mdframed}
\vspace{\algtopspace}
 \captionof{table}{\label{alg:star-steepest-path}
   \starsteepestpath$(T,v,\mathsf{d})$: Returns the steepest path in a
   star graph, with a single non-terminal connected to terminals in
   $T,$ with lengths given by $\mathsf{d}$, and voltages given by $v.$}
    \vspace{\algpostcaptionspace}
    \begin{tight_enumerate}
      \item Sample $t_1$ uniformly and randomly from $T$
      \item Compute 
$t_{2} \in \argmax_{t \in T} \frac{|v(t_1) - v(t)|}{\mathsf{d}(t_{1}) +
  \mathsf{d}(t)}$
\item $\alpha \leftarrow \frac{|v(t_2) - v(t_{1})|}{\mathsf{d}(t_{1}) +
  \mathsf{d}(t_{2})}$
\item Compute $v_{\mathsf{low}} \leftarrow \min_{t \in T} (v(t) + \alpha\cdot {\mathsf
   {d}}(t))$
\item $T_{\mathsf{low}} \leftarrow \{ t \in T\ |\ v(t) > v_{\mathsf{low}} +
  \alpha\cdot \mathsf{{d}}(t)\}$
\item Compute $v_{\mathsf{high}} \leftarrow \max_{t \in T} (v(t) - \alpha\cdot {\mathsf
   {d}}(t))$
\item $T_{\mathsf{high}} \leftarrow \{ t \in T\ |\ v(t) < v_{\mathsf{high}} -
  \alpha\cdot \mathsf{{d}}(t)\}$
\item $T^\prime \leftarrow T_{\mathsf{low}} \cup T_{\mathsf{high}}.$
\item \IF $T^\prime = \emptyset$ 
\item \hspace{\pgmtab} \THEN \IF $v(t_{1}) \ge v(t_{2})$ \THEN \RETURN
  $(t_{1},t_{2})$ \ELSE \RETURN $(t_{2},t_{1})$
\item \hspace{\pgmtab} \ELSE \RETURN \starsteepestpath$(T^{\prime}, v|_{T^{\prime}}, \mathsf{{d}}_{T^{\prime}})$
    \end{tight_enumerate}
\end{mdframed}

\subsection{Faster Lex-minimization}

 \begin{mdframed}
\vspace{\algtopspace}
\captionof{table}{ \label{alg:comp-fast-lex-min} Algorithm \compfastlexmin$(G,v_{0})$: Given a well-posed instance $(G,v_{0}),$ with $T(v_{0}) \neq V_{G}$, outputs  $\lex_G[v_{0}]$.}
    \vspace{\algpostcaptionspace}
    \begin{tight_enumerate}
      \item \WHILE $T(v_{0}) \neq V_{G}$
      \item \hspace{\pgmtab} $v_{0} \leftarrow$ \fixpathsabovepress$(G,v_{0},0)$
      \item \RETURN $v_{0}$
      \end{tight_enumerate}
 \end{mdframed}

 \begin{mdframed}
\vspace{\algtopspace}
\captionof{table}{ \label{alg:fix-paths-above-press} Algorithm
  \fixpathsabovepress$(G,v_{0},\alpha)$: Given a well-posed instance
  $(G,v_{0}),$ with $T(v_{0}) \neq V_{G}$, and a gradient value
  $\alpha,$ iteratively fixes all paths with gradient $> \alpha$.}
\vspace{\algpostcaptionspace}
    \begin{tight_enumerate}
      \item \WHILE $T(v_{0}) \neq V_{G}$
      \item \hspace{\pgmtab} $E_{G} \leftarrow E_{G} \setminus
        (T(v_{0}) \times T(v_{0}))$
      \item \hspace{\pgmtab} Sample uniformly random $e \in E_{G}.$ Let $e = (x_{1},x_{2}).$
      \item \hspace{\pgmtab} Sample uniformly random $x_{3} \in V_{G}.$
      \item \hspace{\pgmtab}  \FOR $i=1$ \TO $3$
      \item \hspace{\pgmtab} \hspace{\pgmtab} $P_i \leftarrow $
        \vertexsteepestpath$(G,v_{0},x_{i})$
        \item  \hspace{\pgmtab}  Let $j \in \argmax_{j \in \{1,2,3\}}
          \nabla P_{j}(v_{0})$
      \item \hspace{\pgmtab}  $G^{\prime} \leftarrow $ \comphighpressgraph$(G,v_{0},\nabla P_{j}(v_{0}))$
      \item \hspace{\pgmtab} \IF $E_{G^{\prime}} = \emptyset,$ 
      \item \hspace{\pgmtab}  \hspace{\pgmtab}  \THEN $v_{0} \leftarrow $ \fixpath$[v_{0},P]$
      \item \hspace{\pgmtab} \hspace{\pgmtab} \ELSE Let
        $G^{\prime}_{i}, i=1,\ldots,r$ be the connected components of $G^{\prime}.$
      \item \hspace{\pgmtab} \hspace{\pgmtab} \hspace{\pgmtab} \FOR
        $i=1,\ldots,r$ 
      \item \hspace{\pgmtab} \hspace{\pgmtab} \hspace{\pgmtab}
        \hspace{\pgmtab} $v_{i} \leftarrow $
        \fixpathsabovepress$(G^{\prime}_{i},v_{0}|_{V_{G^{\prime}_{i}}},\nabla
        P_{j}(v_{0}))$
        \item \hspace{\pgmtab} \hspace{\pgmtab} \hspace{\pgmtab}  \hspace{\pgmtab} \FOR $x
          \in V_{G^{\prime}_{i}},$ set $v_{0}(x) \leftarrow v_{i}(x)$
        \item \hspace{\pgmtab} \IF $\alpha > 0$ \THEN
          $G \leftarrow $\comphighpressgraph$(G,v_{0},\alpha)$
      \item \RETURN $v_{0}$
      \end{tight_enumerate}
 \end{mdframed}


\section{Experiments on WebSpam: Testing More Algorithms}
\label{sec:all-directions-comparison}

For completeness, in this appendix we show how a number of algorithms perform on the web spam experiment of Section~\ref{sec:experiments}. We consider the following algorithms:
\begin{itemize}
\item
\textsc{RandWalk} along in-links. For a detailed description see  \cite{Zhou2007}. This algorithm essentially performs a Personalized PageRank random walk from each vertex $x$ and computes a spam-value for the vertex $x$ by taking a weighted average of the labels of the vertices where the random walk from $x$ terminates.
\emph{Also shown in Section~\ref{sec:experiments}.}
\item
\textsc{DirectedLex}, with edges in the opposite directions of links.
This has the effect that a link to a spam host is evidence of spam,
  and a link from a normal host is evidence of normality.  \emph{Also shown in Section~\ref{sec:experiments}.}
\item
\textsc{RandWalk} along out-links.
\item
\textsc{DirectedLex}, with edges in the directions of links.
This has the effect that a link from to a spam host is evidence of spam,
  and a link to a normal host is evidence of normality.
\item
\textsc{UndirectedLex}: Lex-minimization with links treated as undirected edges.
\item
\textsc{Laplacian}: $l_2$-regression with links treated as undirected edges. 
\item
\textsc{Directed 1-Nearest Neighbor}: Uses shortest distance along paths following out-links. \emph{Spam-ratio} is defined distance from normal hosts, divided by distance to spam hosts. Sites are flagged as spam when spam-ratio exceeds some threshold. We also tried following paths along in-links instead, but that gave much worse results. 
\end{itemize}
We use the experimental setup described in Section~\ref{sec:experiments}. Results are shown in Figure~\ref{fig:alldirs-experiment-image}.
The alternative convention for \textsc{DirectedLex} orients edges in the directions of links. This takes a link from a spam host to be evidence of spam, and a link to a normal host to be evidence of normality. This approach performs significantly worse than our preferred convention, as one would intuitively expect. \textsc{UndirectedLex} and \textsc{Laplacian} approaches also perform significantly worse. \textsc{Directed 1-Nearest Neighbor} performs poorly, demonstrating that \textsc{DirectedLex} is very different from that approach.
As observed by \cite{Zhou2007}, sampling based on a random walk following out-links performs worse than following in-links. Up to 60 \% recall, \textsc{DirectedLex} performs best, both in the regime of 5 \% labels for training and in the regime of 20 \% labels for training.

\begin{figure}[H]
\vspace{-0in}
    \centering
    \includegraphics[scale=0.45]{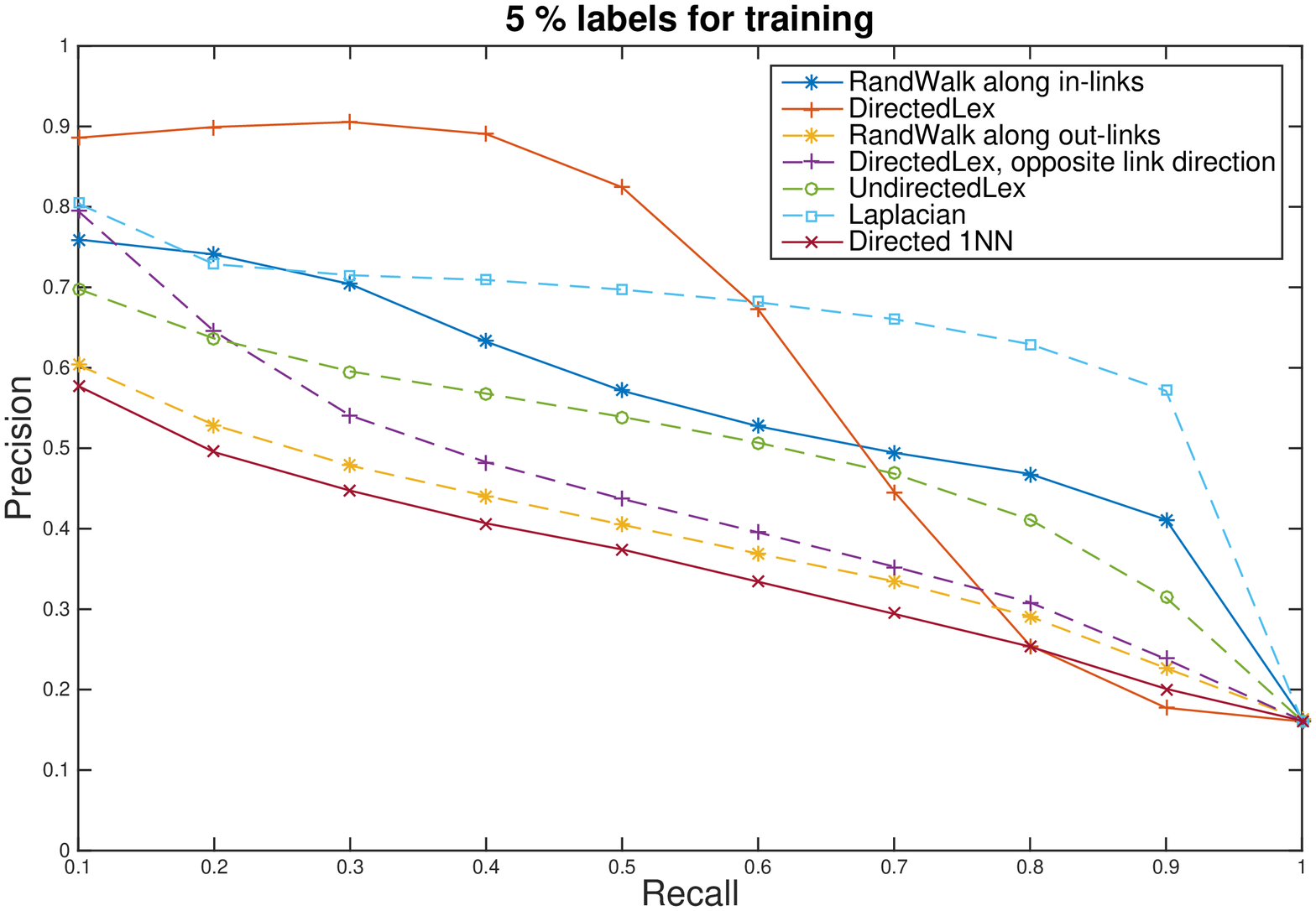}
    \includegraphics[scale=0.45]{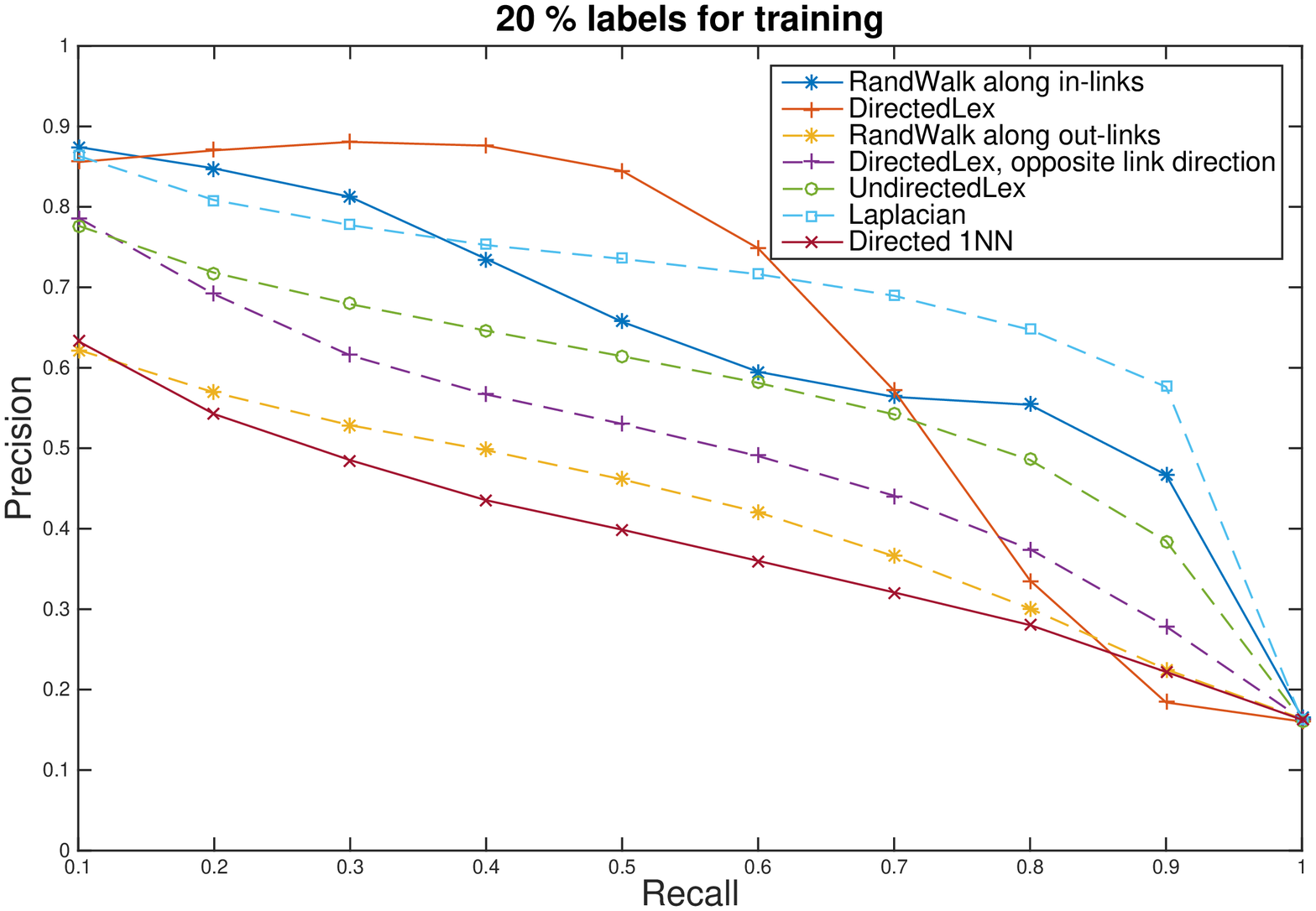}
\vspace{-0in}
\caption{Recall and precision in the WebSpam classification experiment. Each data point shown was computed as an average over 100 runs. The largest standard deviation of the mean precision across the plotted recall values was less than 1.5~\%. The algorithm of \cite{Zhou2007} appears as \textsc{RandWalk} (along in-links). We also show \textsc{RandWalk} along out-links. Our directed lex-minimization algorithm appears as \textsc{DirectedLex}. We also show \textsc{DirectedLex} with link directions reversed,  along with \textsc{UndirectedLex} and \textsc{Laplacian}. }
    \label{fig:alldirs-experiment-image}
\end{figure}
\section{$l_0$-Vertex Regularization Proofs}
\label{sec:l0-proofs}

In this appendix, we prove Theorem~\ref{thm:l0-approx} and Theorem~\ref{thm:l0-exact}. For the purposes of proving the second theorem, we introduce an alternative version of problem~\eqref{eq:outlier-min-fixed-k}. The optimization problem here requires us to minimize $l_0$-regularization budget required to obtain an inf-minimizer with gradient below a given threshold:
\begin{equation}
\begin{aligned}
	 	& \min_{v \in \Reals{n}} \,\,
		 \norm{v(T)-v_{0}(T)}_{0} \\
		& \text{subject to }
		\norm{\grad_{G}[v]}_{\infty}
		\leq \alpha.
\end{aligned}
\label{eq:outlier-min}
\end{equation}
We will also need the following graph construction.
\begin{definition}
The $\alpha$-pressure terminal graph of a partially-labeled graph $(G,v_{0})$ is a directed unweighted graph $G_{\alpha} = (T(v_{0}),\Ehat)$ such that 
$(s,t) \in \Ehat$
if and only if there is a terminal path $P$ from $s$ to $t$ in $G$ with
\[
\nabla P (v_{0}) > \alpha.
\]
\end{definition}
Note that the $\alpha$-pressure terminal graph has $O(n)$ vertices but may be dense, even when $G$ is not.

\begin{mdframed}
\captionof{table}{\label{fig:term-pressure} Algorithm \terminalpressure: Given a well-posed instance $(G,v_{0})$ and $\alpha \geq 0$, outputs $\alpha$ pressure terminal graph $G_{\alpha}$. }

		

          
          Initialize $G_{\alpha}$ with vertex set $V_{\alpha} = T(v_{0})$ and edge set $\Ehat = \emptyset$.
          
          \FOR each terminal $ s \in T(v_{0})$
          \begin{tight_enumerate}
          \item Compute the distances to every other terminal $t$ by running Dijktra's algorithm, allowing shortest paths that run through other terminals.
          \item Use the resulting distances to check for every other terminal $t$ if there is a terminal path $P$ from $s$ to $t$ with
          $\nabla P (v_{0}) > \alpha$. If there is, add edge $(s,t)$ to $\Ehat$.
     
          \end{tight_enumerate}
\end{mdframed}

\begin{lemma}
\label{lem:termpressurealg}
The $\alpha$-pressure terminal graph of a voltage problem $(G,v_{0})$ can be computed in $O( (m+n \log n ) n )$ time using algorithm \terminalpressure \ (Algorithm~\ref{fig:term-pressure}).
\end{lemma}
\begin{proof}
The correctness of the algorithm follows from the fact that Dijkstra's algorithm will identify all shortest distances between the terminals, and the pressure check will ensure that terminal pairs $(s,t)$ are added to $\Ehat$ if and only if they are the endpoints of a terminal path $P$ with $\nabla P (v_{0}) > \alpha$. The running time is dominated by performing Dijkstra's algorithm once for each terminal. A single run of Dijkstra's algorithm takes $O( m+n \log n  )$ time, and this is performed at most $n$ times, for a total running time of $O( (m+n \log n ) n )$.
\end{proof}
We make three observations that will turn out to be crucial for proving Theorems~\ref{thm:l0-approx} and~\ref{thm:l0-exact}.
\begin{observation}
\label{obs:terminalsubmono}
$G_{\alpha}$ is a subgraph of $G_{\beta}$ for $\alpha \geq \beta$. 
\end{observation}
\begin{proof}
Suppose edge  $(s , t)$ appears in $G_{\alpha}$, then for some path $P$
\[
\nabla P (v_{0}) > \alpha \geq \beta,
\]
so the edge also appears in $G_{\beta}$.
\end{proof}

\begin{observation}
\label{obs:TC}
$G_{\alpha}$ is transitively closed. 
\end{observation}
\begin{proof}
Suppose edges  $(s , t)$ and $(t , r)$ appear in $G_{\alpha}$. Let $P_{(s,t)}$, $P_{(t,r)}$, $P_{(s,r)}$ be the respective shortest paths in $G$ between these terminal pairs. Then 

\begin{equation}
\begin{aligned}
\nabla P_{(s,r)} (v_{0}) =
\frac{v_{0}(s) - v_{0}(r)}
{\len(P_{(s,r)})}
\geq
\frac{v_{0}(s) - v_{0}(r)}
{\len(P_{(s,t)}) + \len(P_{(t,r)})}
= \frac{v_{0}(s) - v_{0}(t) +v_{0}(t) - v_{0}(r)}
{\len(P_{(s,t)}) + \len(P_{(t,r)})} \\
\geq
\min
\left( 
\frac{v_{0}(s) - v_{0}(t)}
{\len(P_{(s,t)})}
,
\frac{v_{0}(t) - v_{0}(r)}
{\len(P_{(t,r)})}
\right)
>
\alpha.
\end{aligned}
\end{equation}
So edge  $(s,r)$ also appears in $G_{\alpha}$. This is sufficient for $G_{\alpha}$ to be transitively closed.
\end{proof}

\begin{observation}
\label{obs:DAG}
$G_{\alpha}$ is a directed acyclic graph.
\end{observation}

\begin{proof}
Suppose for a contradiction that a directed cycle appears in $G_{\alpha}$. Let $s$ and  $t$ be two vertices in this cycle. Let $P_{(s,t)}$ and $P_{(t,s)}$ be the respective shortest paths in $G$ between these terminal pairs. Because $G_{\alpha}$ is transitively closed, both edges $(s,t)$ and $(t,s)$ must appear in  $G_{\alpha}$. But $(s,t) \in \Ehat$ implies
\[
v_{0}(s) - v_{0}(t) > \alpha \len(P_{(s,t)}) > 0,
\]
and similarly $(t,s) \in \Ehat$ implies
 \[
v_{0}(t) - v_{0}(s)> \alpha \len(P_{(t,s)}) > 0.
\]
This is a contradiction.
\end{proof}

The usefulness of the $\alpha$-pressure terminal graph is captured in the following lemma. We define a vertex cover of a directed graph to be a vertex set that constitutes a vertex cover in the same graph with all edges taken to be undirected.

\begin{lemma}
\label{lem:vc-l0}
\leavevmode
Given a partially-labeled graph $(G,v_{0})$ and a set $U \subseteq V$,
there exists a voltage assignment $v \in \Reals{n}$ that satisfies
\[
\setof{ t \in T(v_{0}) : v(t) \neq  v_{0}(t)} \subseteq U \text{ and } \norm{\grad_{G}[v]}_{\infty} \leq \alpha,
\]
if and only if $U$ is a vertex cover in the $\alpha$-pressure terminal graph $G_{\alpha}$ of $(G,v_{0})$.
\end{lemma}

\begin{proof}
We first show the ``only if'' direction. Suppose for a contradiction that there exists a voltage assignment $v$ for which $\norm{\grad_{G}[v]}_{\infty} \leq \alpha$, but $U$ is not a vertex cover in $G_{\alpha}$.
Let $(s,t)$ be an edge  $G_{\alpha}$ which is not covered by $U$. The presence of this edge in $G_{\alpha}$ implies that there exists a terminal path $P$ from $s$ to $t$ in $G$ for which
\[
\nabla P (v_{0}) > \alpha.
\]
But, by Lemma~\ref{lem:basics:inf-duality} this means there is no assignment $v$ for $G$ which agrees with $v_{0}$ on $s$ and $t$ and has $\norm{\grad_{G}[v]}_{\infty} \leq \alpha$.
This contradicts our assumption.

Now we show the ``if'' direction. Consider an arbitrary vertex cover $U$ of $G_{\alpha}$. 
Suppose for a contradiction that there does not exist a voltage assignment $v$ for $G$ with $\norm{\grad_{G}[v]}_{\infty} \leq \alpha$ and 
$\setof{ t \in T(v_{0}) : v(t) \neq  v_{0}(t)} \subseteq U$.
Define a partial voltage assignment  $v_{U}$ given by
\begin{equation*}
v_{U}(t) = 
\begin{cases}
v_{0}(t) & \text{ if } t \in T(v_{0}) \setminus U \\
* & \text{ o.w. }
\end{cases}
\end{equation*}
The preceding statement is equivalent to saying that there is no $v$ that extends $v_{U}$ and has  $\norm{\grad_{G}[v]}_{\infty} \leq \alpha$. By Lemma~\ref{lem:basics:inf-duality}, this means there is terminal path between $s,t \in T(v_{U})$ with gradient strictly larger than $\alpha$. But this means an edge $(s,t)$ is present in $G_{\alpha}$ and is not covered. This contradicts our assumption that $U$ is a vertex cover.
\end{proof}

We are now ready to prove Theorem~\ref{thm:l0-exact}.

\begin{proofof}{of Theorem~\ref{thm:l0-exact}}
We describe and prove the algorithm \outlieralg. The algorithm will reduce problem~\eqref{eq:outlier-min-fixed-k} to problem~\eqref{eq:outlier-min}:
Suppose $v^*$ is an optimal assignment for problem~\eqref{eq:outlier-min-fixed-k}. It achieves a maximum gradient
$\alpha^* = \norm{\grad_{G}[v^*]}_{\infty}$.
Using Dijkstra's algorithm we compute the pairwise shortest distances between all terminals in $G$. From these distances and the terminal voltages, we compute the gradient on the shortest path between each terminal pair. By Lemma~\ref{lem:basics:inf-duality}, $\alpha^*$ must equal one of these gradients. So we can solve problem~\eqref{eq:outlier-min-fixed-k} by iterating over the set of gradients between terminals and solving problem~\eqref{eq:outlier-min} for each of these $O(n^2)$ gradients. Among the assignments with $ \norm{v(T)-v_{0}(T)}_{0} \leq k$, we then pick the solution that minimizes $\norm{\grad_{G}[v]}_{\infty}$.

 In fact, we can do better. By Observation~\ref{obs:terminalsubmono}, $G_{\alpha}$ is a subgraph of $G_{\beta}$ for $\alpha \geq \beta$. This means a vertex cover of $G_{\alpha}$ is also a vertex cover of $G_{\beta}$, and hence the minimum vertex cover for $G_{\beta}$ is at least as large as the minimum vertex cover for $G_{\alpha}$. This means we can do a binary search on the set of $O(n^2)$ terminal gradients to find the minimum gradient for which there exists an assignment with $\norm{v(T)-v_{0}(T)}_{0} \leq k$. This way, we only make $O(\log n)$ calls to problem~\eqref{eq:outlier-min}, in order to solve problem~\eqref{eq:outlier-min-fixed-k}.

We use the following algorithm to solve
problem~\eqref{eq:outlier-min}.
\begin{mdframed}
\begin{enumerate}
\parskip0pt
\topsep30pt
\itemsep0pt
\item
\label{alg:outlier:tp}
Compute the $\alpha$-pressure terminal graph $G_{\alpha}$ of $G$ using the algorithm \terminalpressure.

\item
\label{alg:outlier:vc}
Compute a minimum vertex cover $U$ of $G_{\alpha}$ using the algorithm \tcdagvc \ from Theorem~\ref{thm:minvc-tcdag}.

\item
Define a partial voltage assignment  $v_{U}$ given by
\begin{equation*}
v_{U}(t) = 
\begin{cases}
v_{0}(t) & \text{ if } t \in T(v_{0}) \setminus U, \\
* & \text{ otherwise.}
\end{cases}
\end{equation*}
\item
Using Algorithm~\ref{alg:comp-inf-min}, compute voltages $v$ that extend $v_{U}$ and output $v$. 

\end{enumerate}
\end{mdframed}

From Lemma~\ref{lem:termpressurealg}, it follows that step~\ref{alg:outlier:tp} computes the $\alpha$-pressure terminal graph in polynomial time.
From Theorem~\ref{thm:minvc-tcdag} it follows that step~\ref{alg:outlier:vc} computes the a minimum vertex cover of the $\alpha$-pressure terminal graph in polynomial time, because our observations~\ref{obs:TC} and \ref{obs:DAG} establish that the graph is a TC-DAG. From Lemma~\ref{lem:vc-l0} and Theorem~\ref{thm:lineartimeinfmin}, it follows that the output voltages solve program~\eqref{eq:outlier-min}.

\end{proofof}
To prove Theorem~\ref{thm:l0-approx}, we use the standard greedy approximation algorithm for MIN-VC (\cite{Vazirani01}).

\begin{theorem}{\bf 2-Approximation Algorithm for Vertex Cover.}
\label{thm:vc-approx}
The following algorithm gives a 2-approximation to the Minimum Vertex
Cover problem on a graph $G = (V,E)$.
\begin{mdframed}
\begin{enumerate}[start=0]
\parskip0pt
\topsep30pt
\itemsep0pt

\item Initialize $U = \emptyset$.

\item 
\label{alg:greedy-vc:pick}
Pick an edge $(u,v) \in E$ that is not covered by $U$.

\item Add $u$ and $v$ to the set $U$.

\item Repeat from step \ref{alg:greedy-vc:pick} if there are still edges not covered by $U$.

\item Output $U$.

\end{enumerate}
\end{mdframed}

\end{theorem}

We are now in a position to prove Theorem~\ref{thm:l0-approx}

\begin{proofof}{of Theorem~\ref{thm:l0-approx}}
Given an arbitrary $k$ and a partially-labeled graph $(G,v_{0})$, let $\alpha^*$ be the optimum value of program~\eqref{eq:outlier-min-fixed-k}. Observe that by Lemma~\ref{lem:vc-l0}, this implies that $G_{\alpha^*}$ has a vertex cover of size $k$.
Given the partial assignment $v_{0}$, for every vertex set $U$, we define

\begin{equation*}
v_{U}(t) \defeq 
\begin{cases}
v_{0}(t) & \text{ if } t \in T(v_{0}) \setminus U \\
* & \text{ o.w. }
\end{cases}
\end{equation*}

We claim the following algorithm \approxoutlieralg \  outputs a voltage assignment $v$ with $\norm{\grad_{G}[v]}_{\infty} \leq \alpha^*$ and $\norm{v(T)-v_{0}(T)}_{0} \leq 2 k$.

\begin{mdframed}
\vspace{\algtopspace}
\captionof*{table}{ 
Algorithm \approxoutlieralg: \hfill \hfill} 
    \vspace{\algpostcaptionspace}
\begin{enumerate}[start=0]
\parskip0pt
\topsep30pt
\itemsep0pt
\item Initialize $U = \emptyset$.

\item
\label{alg:greedy-l0:pick}
Using the algorithm \steepestpath \ (Algorithm \ref{alg:steepest-path}), find a steepest terminal path in $G$ w.r.t. $v_{U}$.  Denote this path $P$ and let $s$ and $t$ be its terminal endpoints. If there is no terminal path with positive gradient, skip to step~\ref{alg:greedy-l0:volt}.

\item
\label{alg:greedy-l0:add}
Add $s$ and $t$ to the set $U$.

\item
If $\sizeof{U} \leq 2k - 2$ then repeat from step \ref{alg:greedy-l0:pick}.

\item
\label{alg:greedy-l0:volt}
Using the algorithm \compinfmin \ (Algorithm~\ref{alg:comp-inf-min}), compute voltages $v$ that extend $v_{U}$ and output $v$. 

\end{enumerate}
   \end{mdframed}
From the stopping conditions, it is clear that $\sizeof{U} \leq 2k$. If in step~\ref{alg:greedy-l0:pick} we ever find that no terminal paths have positive gradient then our $v$ that extends $v_{U}$ will have $\norm{\grad_{G}[v]}_{\infty} = 0 \leq \alpha^*$, by Lemma~\ref{lem:basics:inf-duality}. Similarly if we find a steepest path with gradient less than $ \alpha^*$ w.r.t. $v_{U}$, then for this $U$ there exists $v$ that extends $v_{U}$ and has $\norm{\grad_{G}[v]}_{\infty}\leq \alpha^*$. This will continue to hold when if we add vertices to $U$. Therefore, for the final $U$, there will exist an $v$ that extends $v_{U}$ and has $\norm{\grad_{G}[v]}_{\infty}\leq \alpha^*$.

If we never find a steepest terminal path $P$ with $\nabla P(v_{0}) \leq \alpha^*$, then each steepest path we find corresponds to an edge in $G_{\alpha^*}$ that is not yet covered by $U$ and our algorithm in fact implements the greedy approximation algorithm for vertex cover described in Theorem~\ref{thm:vc-approx}. This implies that the final $U$ is a vertex cover of $G_{\alpha^*}$ of size at most $2 k$. By Lemma~\ref{lem:vc-l0}, this implies that there exists a voltage assignment $u$ extending $v_{U}$ that has $\norm{\grad_{G}[u]}_{\infty}\leq \alpha^*$. This implies by Theorem~\ref{thm:lineartimeinfmin} that the $v$ we output has $\norm{\grad_{G}[v]}_{\infty}\leq \alpha^*$.

In all cases, the $v$ we output extends $v_{U}$, so $\norm{v(T)-v_{0}(T)}_{0} \leq \sizeof{U} \leq 2 k$.
\end{proofof}

\section{Proof of Hardness of $l_{0}$ regularization for $l_{2}$}
\label{sec:l2_hardness_proof}

We will prove Theorem~\ref{thm:l0_l2_nphard}, by a reduction from minimum bisection.
To this end, let $G = (V,E)$ be any graph.
We will reduce the minimum bisection problem on $G$ to our regularization problem.
Let $n = \sizeof{V}$.
The graph on which we will perform regularization will have vertex set
\[
  V \union \Vhat,
\]
where $\Vhat$ is a set of $n$ vertices that are in $1$-to-$1$ correspondence with $V$.
We assume that every edge in $G$ has weight $1$.
We now connect every vertex in $\Vhat$ to the corresponding vertex in $V$
  by an edge of weight $B$, for some large $B$ to be determined later.
We also connect all of the vertices in $\Vhat$ to each other by edges
  of weight $B^{3}$.
So, we have a complete graph of weight $B^{3}$ edges on $\Vhat$,
  a matching of weight $B$ edges connecting $\Vhat$ to $V$,
  and the original graph $G$ on $V$.
The input potential function will be
\[
  v (a) = \begin{cases}
0 & \text{for $a \in \Vhat$, and}          
\\
1  &  \text{for $a \in V$.}          
\end{cases}
\]
Now set $k = n/2$.
We claim that we will be able to determine the value of the minimum bisection
  from the solution to the regularization problem.

If $S$ is the set of vertices on which $v$ and $w$ differ, then we know that
  the $w$ is harmonic on $S$: for every $a \in S$, $w (a)$ is the weighted
  average of the values at its neighbors.
In the following, we exploit the fact that $\sizeof{S} \leq n/2$.

\begin{claim}\label{clm:vhat}
For every $a \in S \intersect \Vhat$,
  $w (a) \leq 2/n B^{2}$.
\end{claim}
\begin{proof}
Let $a$ be the vertex in  $S \intersect \Vhat$ that maximizes $w (a)$.
So, $a$ is connected to at least $n/2$ neighbors in $\Vhat$ with $w$-value equal to $0$
  by edges of weight $B^{3}$.
On the other hand, $a$ has only one neighbor that is not in $\Vhat $,
  that vertex has $w$-value at most $1$, and it is connected to that vertex
  by an edge of weight $B$.
Call that vertex $c$.
We have
\begin{align*}
  ((n-1)B^{3} + B) w (a) & = B w (c) + \sum_{b \in \Vhat, b \not = a} B^{3} w (b) \\
& = B w (c) + \sum_{b \in \Vhat \intersect S, b \not = a} B^{3} w (b) 
   + \sum_{b \in \Vhat - S} B^{3} w (b)
\\
& \leq  B  + \sum_{b \in \Vhat \intersect S, b \not = a} B^{3} w (a)\\
& \leq  B  + (n/2-1) B^{3} w (a).
\end{align*}
Subtracting $(n/2-1) B^{3} w (a)$ from both sides gives
\[
  (  (n/2)B^{3} + B) w (a) \leq B,
\]
which implies the claim.
\end{proof}

\begin{claim}\label{clm:V}
For $a \in S \intersect V$, $w (a) \leq n / B$.
\end{claim}
\begin{proof}
Vertex $a$ has exactly one neighbor in $\Vhat$.
Let's call that neighbor $c$.
We know that $w (c) \leq 2 / B^{2} n$.
On the other hand, vertex $a$ has fewer than $n-1$ neighbors
  in $V$, and each of these have $w$-value at most $1$.
Let $d_{a}$ denote the degree of $a$ in $G$.
Then,
\begin{align*}
  (B + d_{a}) w (a)
  &  \leq d_{a} + B \frac{2}{ B^{2} n}.
\end{align*}
So,
\begin{align*}
  w (a)
  &  \leq  \frac{d_{a} + 2/ B n}{ d_{a} + B } \\
  & \leq \frac{n + (2 / B n)}{B + n} \\
  &  \leq n / B.
\end{align*}
%
\end{proof}

We now estimate the value of the regularized objective function.
To this end, we assume that 
\[
\sizeof{S} = k = n/2.
\]
Let 
\[
T = S \intersect V,
\]
and
\[
 t = \sizeof{T}.
\]
We will prove that $S \subset V$ and so $S = T$ and $t = n/2$.

Let $\delta$ denote the number of edges on the boundary of $T$ in $V$.
Once we know that $t = n/2$, $\delta$ is the size of a bisection.

\begin{claim}\label{clm:crossEdges}
The contribution of the edges between $V$ and $\Vhat$ to the objective
  function is at least
\[
(n-t) B - 4/B
\]
and at most
\[
(n-t) B +  tn^2/B.
\]
\end{claim}
\begin{proof}
For the lower bound, we just count the edges between vertices in $V \setminus T$
  and $\Vhat$.
There are $n-t$ of these edges, and each of them has weight $B$.
The endpoint in $V \setminus T$ has $w$-value $1$,
  and the endpoint in $\Vhat$ has $w$-value at most
  $2/nB^{2}$.
So, the contribution of these edges is at least
\[
  (n-t) B (1-2/n B^{2})^{2} 
  \geq (n-t) B (1 - 4/n B^{2})
  \geq (n-t) B - 4/B.
\]
For the upper bound, we observe that the difference in $w$-values across each of these $n-t$ 
  edges is at most $1$, so their
  total contribution is at most
\[
  (n-t) B.
\]
Since for every vertex  $a \in T$, $w(a) \leq n/B$, and also every vertex $b \in \Vhat$, $w(b) \leq 2/nB^2$, the contribution due to edges between $T$ and $\Vhat$ is at most 
$$ t (n/B)^2B = tn^2/B.$$
 
\end{proof}

We will see that this is the dominant term in the objective function.
The next-most important term comes from the edges in $G$.

\begin{claim}\label{clm:Gedges}
The contribution of the edges in $G$ to the objective function
  is at least
\[
\delta (1 - 2n/B)
\]
and at most
\[
\delta +   (t^{2}/2) (n/B)^{2}
\]
\end{claim}
\begin{proof}
Let $(a,b) \in E$.
If neither $a$ nor $b$ is in $T$, then $w (a) = w (b) = 1$,
  and so this edge has no contribution.
If $a \in T$ but $b \not \in T$, then the difference in $w$-values on them is between
 $(1-n/B)$ and $1$.
So, the contribution of such edges to the objective function
  is between
\[
  \delta (1 - 2n/B) \text{ and } \delta .
\]
Finally, if $a$ and $b$ are in $T$, then the difference in $w$-values on them is at most $n/B$, and so the contribution
  of all such edges to the objective function is at most
\[
  (t^{2}/2) (n/B)^{2}.
\]
\end{proof}

\begin{claim}\label{clm:hathat}
The edges between pairs of vertices in $\Vhat$
  contribute at most $2/B$ to the objective function.
\end{claim}
\begin{proof}
As $0 \leq w (a) \leq 2 / B^{2} n$ for every $a \in \Vhat$,
  every edge between two vertices in $\Vhat$ can contribute at most
\[
  B^{3} (2 / B^{2} n)^{2} = 4 / B n^{2}.
\]
As there are fewer than $n^{2}/2$ such edges, their total contribution to the
  objective function is at most
\[
  (n^{2}/2) (4 / B n^{2}) = 2/B.
\]
\end{proof}

\begin{lemma}\label{lem:obj}
If $n \geq 4$ and  $B = 2 n^{3}$,
the value of the objective function is at least
\[
(n-t) B + \delta - 1/2
\]
and at most
\[
(n-t) B + \delta + 1/3.
\]
\end{lemma}
\begin{proof}
Summing the contributions in the preceding three claims,
  we see
 that the value of the objective function is at least
\begin{align*}
(n-t) B - 4/B
+
\delta (1 - 2n/B)
& \geq
(n-t) B + \delta - 4/B - 2 n \delta /B
\\
&
\geq
(n-t) B + \delta - n^{3} / B
\\
& \geq 
(n-t) B + \delta - 1/2,
\end{align*}
as $\delta \leq (n/2)^{2}$.

Similarly, the objective function is at most
\begin{align*}
(n-t) B +  tn^2/B +
\delta +   (t^{2}/2) (n/B)^{2}
+
 2/B
& \leq
(n-t)B +  n^3/2B + \delta 
+ n^{4} / 8 B^{2}
 + 2 / B
\\
& \leq 
(n-t)B +  n^3/2B + \delta 
+ 1 / 32 n^{2}
+ 1/n^{3}
\\
& \leq 
(n-t)B +  \delta + 1/3.
\end{align*}
\end{proof}

\begin{claim}\label{clm:subV}
If $n \geq 2$ and $B = 2n^{3}$, then
$S \subset V$.
\end{claim}
\begin{proof}
The objective function is minimized by making $t$ as large as possible,
  so $t = n/2$ and $S \subset V$.
\end{proof}

\begin{theorem}\label{thm:minbis}
The value of the objective function reveals the value of the minimum bisection in $G$.
\end{theorem}
\begin{proof}
The value of the objective function will be between
\[
(n/2) B + \delta - 1/2
\]
and
\[
 (n/2) B + \delta +1/3.
\]
So, the objective function will be smallest when $\delta$ is as small as possible.
\end{proof}

Theorem~\ref{thm:minbis} immediately implies Theorem~\ref{thm:l0_l2_nphard}.

\end{document}